\documentclass{article}

\usepackage[preprint]{neurips_2026}

\usepackage[utf8]{inputenc} 
\usepackage[T1]{fontenc}    
\usepackage{hyperref}       
\usepackage{url}            
\usepackage{booktabs}       
\usepackage{amsfonts}       
\usepackage{nicefrac}       
\usepackage{microtype}      
\usepackage{xcolor}         
\usepackage{graphicx}
\usepackage{amsmath, amssymb, amsthm}

\usepackage{algorithm}
\usepackage{algpseudocode}

\usepackage{multirow}

\usepackage{caption}

\theoremstyle{plain}

\newtheorem{lemma}{Lemma}
\newtheorem{proposition}{Proposition}

\theoremstyle{definition}
\newtheorem{definition}{Definition}

\theoremstyle{remark}
\newtheorem{remark}{Remark}

\theoremstyle{plain}
\newtheorem{corollary}{Corollary}

\usepackage{titlesec}
\titlespacing*{\paragraph}{0pt}{0.5ex plus 0.2ex minus 0.2ex}{1em}

\title{Distribution Shift in Missing Data Imputation: A Risk-Based Perspective and Importance-Weighted Correction under MAR}

\author{%
  Luke Shannon\\
  School of Mathematics\\
  University of Bristol\\
  Bristol, England \\
  \texttt{nq24207@bristol.ac.uk} \\
  \And
  Song Liu\\
  School of Mathematics\\
  University of Bristol\\
  Bristol, England \\
  \texttt{song.liu@bristol.ac.uk}
  \And
  Katarzyna Reluga\\
  School of Business and Economics\\
  Humboldt University of Berlin\\
  Berlin,  Germany\\
  \texttt{katarzyna.reluga@hu-berlin.de}
}

\begin{document}

\maketitle


\begin{abstract}
Missing data imputation, where a model is trained on observed data to estimate unobserved values, is a fundamental problem in machine learning.
In this paper, we rigorously formulate imputation model learning as a mean-squared error risk minimisation problem. We show that when the probability of missingness depends on the data, many state-of-the-art methods fail to account for the resulting distribution shift between the observed data used for training and the full data distribution used for evaluation. Consequently, these approaches do not minimise mean-squared error on the full data distribution.
Instead, we propose a novel imputation algorithm designed to learn an imputation model from the observed data while explicitly accounting for this distribution shift.
Simulation studies show consistent improvements over otherwise identical uncorrected baselines, with average reductions of 3\% in RMSE and 7\% in Wasserstein distance.
\end{abstract}

\section{Introduction}\label{sec:Introduction}

Missing data poses a fundamental obstacle to statistical inference in applied settings. Most classical estimators and modern machine learning methods are designed for complete data and break down in the presence of missingness. A well-established approach, known as imputation, is to complete the dataset by estimating the missing entries, thereby enabling the application of complete-data methods. A wide range of imputation approaches have been proposed, including generative models such as GAIN and MIWAE \citep{yoon2018gainmissingdataimputation, MatteiEtAlMIWAE}, optimal transport–based methods such as Sinkhorn \citep{MuzellecEtAllSinkhorn}, low-rank matrix completion techniques \citep{MazumderEtAlSoftimpute},  iterative procedures such as the \textbf{e}xpectation-\textbf{m}aximisation algorithm (EM) \citep{dempster1977maximum}, and round robin iterative procedures such as \textbf{m}ultiple \textbf{i}mputation by \textbf{c}hained \textbf{e}quations (MICE) \citep{BuurenMICE}, MissForest \citep{StekhovenEtAlMissForest} and HyperImpute \citep{jarrett2022hyperimpute}.


We consider the problem of learning an imputation model from observed data that accurately recovers unobserved entries. We adopt a risk minimisation approach to this \citep{vapnik1998statistical}, where model performance is measured via a loss comparing true and imputed values. However, the risk, as defined with respect to the true data, has an implicit dependence on unobserved entries, which are not available at learning time. Nevertheless, under the assumption that the missingness mechanism depends only on observed variables, we show that the risk can be  written as an importance-weighted expectation over the observed data. These weights highlight a distributional shift between the observed  data used for training and the full data distribution used for evaluation, which has not been explicitly accounted for in existing risk-based imputation approaches. Building on our theoretical results, we propose an importance-weighted risk minimisation algorithm that explicitly accounts for the distribution shift between observed data and full data distributions.

\paragraph{Related Work: }Many existing imputation methods, including likelihood-based, iterative, and adversarial approaches, can be interpreted within our risk minimisation framework, enabling direct comparison with our approach. Likelihood-based methods such as EM and MIWAE minimise an observed-data negative log-likelihood risk; round-robin iterative approaches such as MICE, MissForest, and HyperImpute optimise a sequence of reconstruction-based risks defined on observed entries; and GAIN minimises a reconstruction loss on observed data under an adversarial regularisation scheme. While these methods differ in modelling assumptions and optimisation procedures, they all share a reliance on observed data and the absence of an explicit correction for the resulting mismatch with the full data distribution. In contrast, our approach explicitly accounts for this distributional shift within a risk minimisation framework via importance weighting. We place our findings in a broader literature context in Section~\ref{sec:Discussion}.

Our contributions are as follows:
\begin{itemize}
    \item In Section~\ref{sec:MSERisk}, we show the imputation risk can be expressed as an importance-weighted expectation over observed data, making explicit the effect of distributional shift on the risk.

    \item In Section~\ref{sec:OurAlgorithm}, we propose an imputation algorithm that incorporates our derived importance weights in the training objective. This approach is illustrated Figure~\ref{fig:MotivatingExample}.

    \item In Section~\ref{sec:Experiments},  we benchmark our proposed importance-weighted algorithm against a range of state-of-the-art imputation methods. 
\end{itemize}

\begin{figure}[t]
  \centering  \includegraphics[width=0.8\linewidth]{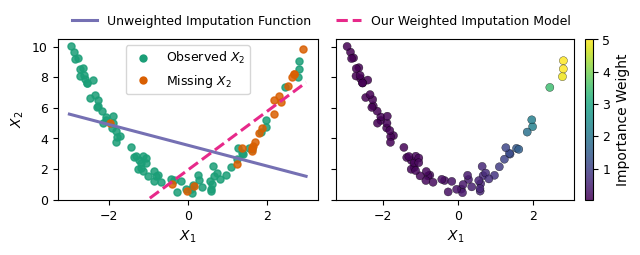}
  \captionsetup{skip=4pt}
  \caption{
The left panel illustrates a two-dimensional dataset where missingness in $X_2$ depends on the fully observed variable $X_1$, such that points with higher values of $X_1$ are more likely to have missing values in $X_2$, inducing a distribution shift between the observed data (green points) and the full data distribution. As a result, an imputation function learned without correcting for the  selection bias (blue line) is less accurate than a weighted approach (dashed red line) in predicting the missing values (red points). The weighted approach uses the importance weights (right panel) in its training to correct for the selection bias. The plot is best viewed in colour.
  }
  \label{fig:MotivatingExample}
\end{figure}

\section{Problem Setup and Background}

To formalise our approach, we introduce the notation for modelling missing data, and define imputation functions together with their associated risk under the full data-generating distribution.

\subsection{Missing Data Notation and Mechanisms}\label{sec:missingness}

Throughout the paper, uppercase letters denote random variables and lowercase letters their realisations, with subscripts $i$ and $\neg i$ denoting coordinate-wise and complement indexing respectively. Let $X = (X_1, \ldots, X_d) \in \mathbb{R}^d$ denote a $d$-dimensional random vector drawn from a distribution $P_X$ with density $p(x)$, representing the true joint distribution of the data. Let $\mathcal{X} := \mathcal{X}_1 \times \cdots \times \mathcal{X}_d$ denote its support, where each $\mathcal{X}_i \subset \mathbb{R}$ is either a finite set (if $X_i$ is discrete) or a possibly disconnected subset of $\mathbb{R}$ (if $X_i$ is continuous), for $i = 1, \ldots, d$. We assume throughout that the data satisfies $\mathbb{E}[X_i^2] < \infty$ for all variables $i$, ensuring that all risks considered in this work are well-defined.

In practice, we do not observe $X$ directly, but instead observe a partially missing version determined by a missingness pattern. Let $R = (R_1, \dots, R_d) \in \{0,1\}^d$ denote the missingness pattern drawn from a distribution $P_R$ with density $p(r)$,  where $R_i = 1$ indicates that $X_i$ is observed and $R_i = 0$ that it is missing. We assume $(X,R)$ is distributed according to $P_{X,R}$ with joint density $p(x,r)$. We model missingness by introducing a symbol $* \notin \bigcup_{i=1}^d \mathcal{X}_i$ to denote missing values, and introduce the partially observed vector $\tilde{X} \in \tilde{\mathcal{X}} := \prod_{i=1}^d \tilde{\mathcal{X}}_i$, where $\tilde{\mathcal{X}}_i := \mathcal{X}_i \cup \{*\}$. $\tilde{X}$ is defined component-wise by
$\tilde{X}_i = X_i$ if $R_i = 1$ and $\tilde{X}_i = *$ if $R_i = 0$
Draws from $\tilde{X}$ correspond to the partially observed data available in practice. For any realisation $r \in \{0,1\}^d$, the vector $X$ induces a decomposition into observed and missing components $X_{\mathrm{obs}} = (X_j : r_j = 1)$ and $X_{\mathrm{miss}} = (X_j : r_j = 0)$.

Finally, the joint density factorises as $p(x,r) = p(r \mid x)p(x)$, defining the missing data mechanism. We distinguish three standard regimes: Missing Completely At Random (\textbf{MCAR}), where $p(r \mid x) = p(r)$; Missing At Random (\textbf{MAR}), where $p(r \mid x) = p(r \mid x_{\mathrm{obs}})$; and Missing Not At Random (\textbf{MNAR}), otherwise.

We use “density” throughout to refer to both probability density and mass functions, and assume all distributions are absolutely continuous with respect to the Lebesgue measure for continuous variables and the counting measure for discrete variables.

\subsection{Imputation Functions and their Risk}\label{sec:ImpFuncAndRisk}
We cast the imputation problem within a risk minimisation framework \citep{vapnik1998statistical}, aiming to learn an imputation function that minimises discrepancy between imputed and true data. Mean squared error (MSE) is a natural choice of discrepancy measure due to its interpretability and widespread use \citep{yoon2018gainmissingdataimputation, StekhovenEtAlMissForest}. Accordingly, we focus on MSE risk minimisation.  We formalise the risk via a functional $\mathcal{J}: \mathcal{G} \to \mathbb{R}$ defined over a hypothesis class $\mathcal{G}$ of measurable imputation functions with finite variance, ensuring that $\mathcal{J}(g)$ is well-defined and finite for all $g \in \mathcal{G}$. We now define the central objects of our study.
\begin{definition}[Imputation function]
\label{def:imputation-map}
An imputation function $g : \tilde{\mathcal{X}} \to \mathcal{X}$ maps a partially observed vector $\tilde{x}$ to a complete vector $g(\tilde{x}) = \big( g_i(\tilde{x}) : i \in \{1, \ldots, d\} \big)$, 
subject to the constraint that observed entries are left unchanged. That is, for each coordinate $i$ such that $r_{\tilde{x}, i} = 1$, where $r_{\tilde{x}}$ denotes the missingness pattern associated with $\tilde{x}$, we require $g_i(\tilde{x}) = \tilde{x}_i$.
\end{definition}

\begin{definition}[Imputation risk]
The MSE risk of an imputation function $g$, given in Definition~\ref{def:imputation-map}, is
\begin{equation}\label{eq:GeneralImputationProblem}
\mathcal{J}(g) = \mathbb{E}_{X, R}\Big[
\mathrm{MSE}\big\{ g(\tilde{X}), X\big\}\Big],
\end{equation}
where $\tilde{X}$ is induced by $(X, R)$ as described in Section~\ref{sec:missingness}.
\end{definition} 

Let $\mathcal{P} := \{\, i \in \{1,\dots,d\} : p(R_i = 0) > 0 \,\}$ denote the set of coordinates subject to missingness. Under MSE loss, the imputation risk decomposes additively over $\mathcal{P}$ into a collection of univariate regression problems. We restrict attention to the  decomposed form of the risk given in Lemma~\ref{lemma:MSERiskUnObs} for the remainder of the paper; a derivation is provided in Appendix~\ref{app:MSELossDecomposes}.

\begin{lemma}
\label{lemma:MSERiskUnObs} 
For a deterministic imputation map given in Definition~\ref{def:imputation-map}, the risk in equation~\eqref{eq:GeneralImputationProblem} satisfies
\begin{equation}\label{eq:MSERiskUnobsTotal}
    \mathcal{J}(g) = \sum_{i \in \mathcal{P}} p(R_i = 0)\, \mathcal{J}_i(g_i),
\end{equation}
where
\begin{equation}\label{eq:MSERiskUnobsComponent}
\mathcal{J}_i(g_i) = \mathbb{E}_{X, R_{\neg i}}\!\left[\left\{g_i(\tilde{X}) - X_i\right\}^2 \mid R_i = 0\right].
\end{equation}
\end{lemma}

\subsection{Covariate Shift}\label{sec:CovariateShift}
Covariate shift is central to the theory developed in this paper. Lemma~\ref{lemma:MSERiskUnObs} shows that the imputation risk decomposes into a collection of univariate regression problems. In Section~\ref{sec:MSERisk}, we show that the univariate regression problems arising from Lemma~\ref{lemma:MSERiskUnObs} admit a covariate shift formulation under the observed data distribution. To formalise the connection, we recall the standard covariate shift setting. Consider a regression problem in which we aim to predict an output $y$ from a predictor $x$. Let $p_{\mathrm{train}}(x,y)$ and $p_{\mathrm{test}}(x,y)$ denote the joint distributions of the training and testing data respectively. Covariate shift refers to the setting in which the input marginal distributions differ,
$p_{\mathrm{train}}(x) \neq p_{\mathrm{test}}(x)$, while the conditional distribution remains invariant,
$p_{\mathrm{train}}(y \mid x) = p_{\mathrm{test}}(y \mid x)$. In covariate shift settings, empirical risk minimisation on the training distribution is biased with respect to the test distribution. A standard correction is importance weighting, which reweights training samples by the density ratio $p_{\mathrm{test}}(x) / p_{\mathrm{train}}(x)$ in order to recover an unbiased estimate of the test risk \citep{Shimodaira2000CovariateShift, SugiyamaEtAlCovariateShiftCV}.

\section{Risk Minimisation Formulation for Imputation}\label{sec:RiskMinFormulation}

In Section~\ref{sec:MSERisk}, we derive an observed data representation of the coordinate-wise imputation risks from Lemma~\ref{lemma:MSERiskUnObs}, revealing a connection to covariate shift. In Sections~\ref{sec:SurrogateRisk} and \ref{sec:OurAlgorithm}, we build on this representation to develop the proposed imputation algorithm.

\subsection{The Oracle Risk Minimisation Problem}\label{sec:MSERisk}

The coordinate-wise risks $\mathcal{J}_i(g_i)$, defined in equation~\eqref{eq:MSERiskUnobsComponent}, are intractable since they depend on the unobserved quantity $X_i \mid R_i = 0$. In Proposition~\ref{prop:GeneralImputationObs}, we show that, under the MAR assumption and  overlap condition defined in Definition~\ref{def:overlap}, each $\mathcal{J}_i(g_i)$ can be equivalently written as importance-weighted expectation over observed data. The proof of Proposition~\ref{prop:GeneralImputationObs} is given in Appendix~\ref{app:rewritingMSEobjective}.

\begin{definition}[Overlap]\label{def:overlap}
For $i \in \mathcal{P}$, we say that the overlap condition holds for $R_i$ if
\[
0 < p(R_i = 1 \mid X_{\mathrm{obs}} = x_{\mathrm{obs}}, R_{\neg i} = r_{\neg i}) < 1,
\]
for all $(x_{\mathrm{obs}}, r_{\neg i})$ such that $p(x_{\mathrm{obs}}, r_{\neg i} \mid R_i = 0) > 0$.
\end{definition}

\begin{proposition}\label{prop:GeneralImputationObs}
For $i \in \mathcal{P}$, assume that the MAR assumption holds, so that $p(r \mid x) = p(r \mid x_{\mathrm{obs}})$, and that $R_i$ satisfies the overlap condition in Definition~\ref{def:overlap}. Then the functional $\mathcal{J}_i(g_i)$, defined in equation~\eqref{eq:MSERiskUnobsComponent}, can be expressed in terms of observed data as:
\begin{equation}
    \mathcal{J}_i(g_i) =
    \mathbb{E}_{X, R_{\neg i}}\!\Big[
    w_i(X_{\mathrm{obs}}, R_{\neg i})
    \left\{ g_i(\tilde{X}) - X_i \right\}^2
    \mid R_i = 1
    \Big].\label{eq:MSELossObs}
\end{equation}

where the weights $w_i(\cdot,\cdot)$ are defined by
\begin{align}
\label{eq:importanceWeights}
    w_i(x_{\mathrm{obs}}, r_{\neg i})
    := \frac{p(x_{\mathrm{obs}}, r_{\neg i} \mid r_i = 0)}{p(x_{\mathrm{obs}}, r_{\neg i} \mid r_i = 1)}.
\end{align}
\end{proposition}

Proposition~\ref{prop:GeneralImputationObs} reveals an interesting phenomenon. Optimal imputation can be characterised as solving a  collection of covariate shift problems, one for each partially observed coordinate in the data, making explicit a link between the fields of missing data imputation and covariate shift. Conditioning on the missingness indicator $R_i$ induces two subpopulations: the observed entries ($R_i = 1$) and the unobserved entries ($R_i = 0$). Under MAR, the conditional distribution $X_i \mid (X_{\mathrm{obs}}, R_{\neg i})$ is invariant across these subpopulations, while the marginal distribution of $(X_{\mathrm{obs}}, R_{\neg i})$ are not necessarily the same. The $i^{th}$ coordinate risk in equation~\eqref{eq:MSELossObs} is therefore analogous to covariate shift, as introduced in Section~\ref{sec:CovariateShift}, where $X_i$ plays the role of the response variable, $(X_{\mathrm{obs}}, R_{\neg i})$ are the covariates, and the observed and unobserved subpopulations correspond to training and test distributions respectively. The weights in equation~\eqref{eq:importanceWeights}, analogous to $p_{\mathrm{test}}(x)/p_{\mathrm{train}}(x)$, reweight the observed data so that empirical risk minimisation targets the distribution of $X_i$ in the unobserved subpopulation.

In contrast, existing risk based imputation approaches can be viewed as minimising an approximation of equation~\eqref{eq:MSELossObs} in which the weighting term is omitted. Such an approach implicitly fails to account for the the distributional shift between observed and full data distributions, yielding imputation models that are optimised for the wrong distribution.

\subsection{Our Surrogate Risk Objective}\label{sec:SurrogateRisk}
Although the oracle coordinate-wise risk derived in equation~\eqref{eq:MSELossObs} provides a principled target for imputation, both $g_i$ and $w_i$ depend on $X_{\mathrm{obs}}$ and $R_{\neg i}$. The weights $w_i$ depend on these quantities explicitly, while $g_i$ depends on them implicitly through its input $\tilde{X}$. This dependence has two consequences. First, conditioning on $R_{\neg i}$ increases the effective dimensionality of both $g_i$ and $w_i$, since each partially observed variable contributes additional pattern-specific features, which becomes particularly pronounced in settings with many variables subject to missingness. Second, the structure of $X_{\mathrm{obs}}$ itself varies across missingness patterns, leading to heterogeneity in the resulting optimisation problems across coordinates and patterns. As a consequence, a naive minimisation of equation~\eqref{eq:MSELossObs} would require solving a collection of pattern-specific risk minimisation problems, which is computationally and statistically inefficient \citep{Stempfle_2023}. We therefore seek an approach that avoids the need for pattern-specific conditional models while maintaining a low-dimensional conditioning set for each model.

\paragraph{Low dimension modelling set:} Under the MAR assumption, the missingness indicators $R_{\neg i}$ are conditionally independent of $X_i$ given $X_{\mathrm{obs}}$, implying that they do not provide additional information for the imputation task beyond $X_{\mathrm{obs}}$. The conditional independence provides a principled justification for restricting attention to imputation models in which both $g_i$ and  $w_i$ depend only on $X_{\mathrm{obs}}$, yielding a lower dimensional form of our optimisation problem without loss of relevant information. 

\paragraph{Avoiding pattern specific modelling:}
To overcome the need for pattern-specific modelling, we adopt a round robin iterative strategy \citep{BuurenMICE, StekhovenEtAlMissForest}. In such approaches, missing data is typically initialised using  column-wise means. The algorithm then proceeds by cycling through partially observed variables in turn. At each step, a conditional imputation model is fitted for the target variable, using the current values of the remaining variables as predictors. Missing entries are then updated using the fitted model. A key advantage of this approach is that each iteration operates on a fully imputed dataset, thereby avoiding the need to construct separate models for different missingness patterns. Round-robin imputation approaches are widely used in practice and across applied fields such as epidemiology \citep{RescheRigonEtAl2013} and psychology \citep{XuEtAl2020} due to their flexibility, interpretability, and robust empirical performance --- advantages that we retain by adopting this framework.

\paragraph{Surrogate risk formulation:}
We therefore propose minimising a data-dependent surrogate risk within a round robin scheme. At iteration $t$, we treat the current imputed dataset $\hat{X}^{(t)}$ as given. When updating coordinate $i$, the goal is to select an imputation function $g_i$ that performs well with respect to this current state of the data. To this end, we define the surrogate objective
\begin{align}
\mathcal{J}_{i}^{(t)}(g_i)
= \mathbb{E}\!\left[
    w_i(\hat{X}^{(t)}_{\neg i})
    \big( g_i(\hat{X}^{(t)}_{\neg i}) - X_i \big)^2
    \mid R_i = 1
\right].\label{eq:SurrogateRisk}
\end{align}

which will form the basis of our approach. 

\section{Our Weighted Imputation Algorithm}\label{sec:OurAlgorithm}

\begin{algorithm}[tb]
\caption{Weighted Iterative Imputation}
\label{alg:weighted-imputation}
\begin{algorithmic}
    \State {\bfseries Input:} Incomplete dataset $\mathcal{D}^{\mathrm{miss}}$, visitation order $v$, convergence criteria $\delta$. 
    \State {\bfseries Output:} Completed dataset $\hat{\mathcal{D}}$.
    \State $\hat{\mathcal{D}} \gets \texttt{Initial Imputation}(\mathcal{D})$
    \Repeat
        \For{column $i$ {\bfseries in} $v$}
            \State
            Estimate $ w_i$ using Algorithm~\ref{alg:estimate-weights}.
            \State Fit $g_i$ to minimise \eqref{eq:SurrogateRisk} using estimated $w_i$ as the weights.
            \State Replace missing entries in column $i$ of $\hat{\mathcal{D}}$ with predictions from $\hat{g}_i$.
        \EndFor
    \Until{Convergence criteria $\delta$ is satisfied}
    \State \textbf{return} $\hat{\mathcal{D}}$
\end{algorithmic}
\end{algorithm}

\begin{algorithm}[tb]
\caption{Estimate weights $w_i$ for column $i$}
\label{alg:estimate-weights}
\begin{algorithmic}
    \State {\bfseries Input:} Imputed dataset $\hat{\mathcal{D}} = \{\hat{x}^{(k)}, r^{(k)}\}_{k=1}^N$, target column $i$. 
    \State {\bfseries Output:} Estimated weights $\{w_i^{(k)}\}_{k=1}^N$ for column $i$.
    \State {\bfseries Procedure:}
    \State Model $p(R_i = 1 \mid \hat x_{\neg i})$ using $\eta_i(\hat x_{\neg i})$ trained using binary classification loss.
    \State Compute weights: $w_i^{(k)} \gets \frac{1 - \eta_i\left(\hat{x}^{(k)}_{\neg i}\right)}{\eta_i\left(\hat{x}^{(k)}_{\neg i}\right)}$.
    \State \textbf{return} $\{w_i^{(k)}\}_{k=1}^N$
\end{algorithmic}
\end{algorithm}

We now describe our proposed weighted iterative imputation algorithm. The algorithm proceeds in a manner analogous to existing round-robin iterative approaches \citep{BuurenMICE, StekhovenEtAlMissForest, jarrett2022hyperimpute}. The key difference is that the importance weights $w_i$ in the surrogate risk \eqref{eq:SurrogateRisk} are unknown and must also be estimated from data at each iteration and incorporated into the estimation of the conditional imputation model $g_i$.

As in existing approaches, the practitioner specifies a conditional model hypothesis class $\mathcal{G}_i$ for each partially observed coordinate $i \in \mathcal{P}$, as defined in Section~\ref{sec:ImpFuncAndRisk}, from which the imputation function $g_i$ is selected. In addition, we introduce a corresponding conditional model class $\mathcal{W}_i$, from which the importance-weighting functions are estimated at each iteration.

Algorithm~\ref{alg:weighted-imputation} implements the surrogate risk minimisation procedure described in Section~\ref{sec:SurrogateRisk}. The algorithm iterates through the coordinates according to a visitation order $v$, which is an ordered subset of $\mathcal{P}$. At iteration $t$ when updating column $i$, the current imputed dataset provides an empirical realisation of $\hat{X}^{(t)}$ given in equation~\eqref{eq:SurrogateRisk}. Given this state, importance weights are estimated using a binary classification model selected from the hypothesis class $\mathcal{W}_i$, following a density ratio estimation approach \citep{QinBinaryDRE, BickelBinaryDRE}, as detailed in Algorithm~\ref{alg:estimate-weights}. The resulting weights are then used to fit the conditional imputation function $g_i \in \mathcal{G}_i$ by minimising an empirical approximation of \eqref{eq:SurrogateRisk}, after which missing entries in column $i$ are updated while observed entries remain unchanged. The procedure is repeated until the stopping criterion $\delta$ is satisfied.

\subsection{Practical Implementation}\label{sec:PracticalImplementation}

We now describe the implementation details of Algorithms~\ref{alg:weighted-imputation} and \ref{alg:estimate-weights}. When updating column $i \in \mathcal{P}$, we proceed as follows.

\paragraph{Weight Estimation:} 
Following Algorithm~\ref{alg:estimate-weights}, we estimate importance weights by modelling the missingness mechanism $p(R_i = 1 \mid X_{\neg i})$ using a binary classification model from a hypothesis class $\mathcal{W}_i$ of logistic regression models. Hyperparameters are tuned over a predefined grid on a random subsample of the data, to minimise standardised mean differences (SMD) \citep{AustinEtAlIPTW} between observed and missing values. 
The estimated probabilities are transformed into importance weights and clipped at the 5th and 95th percentiles to reduce the influence of extreme values.
\paragraph{Stabilisation Measures:}
Incorporating the estimated weights can reduce the effective variance of covariates under the weighted distribution, leading to unstable conditional model estimates. To mitigate this, we identify and remove, at each iteration, features with variance below $10^{-12}$. This filtering step is performed prior to conditional model fitting.

\paragraph{Regression Estimation:} 
Given the observed data, estimated weights, and the stabilised set of covariates, we fit a weighted regression model. 
To control variance, we follow \citep{SugiyamaEtAlCovariateShiftCV} and introduce a tempering parameter $\gamma$. The latter is selected via cross-validation on a random subset of the observed data for column $i$. 
The final model is then estimated using weighted least squares with weights $w_i^{\gamma}$. 
Since tuning $\gamma$ at each iteration represents the primary computational cost of our procedure, we adopt a warm-starting strategy to reduce overhead and improve stability across iterations. 
Further details are provided in Appendix~\ref{appendix:gamma-discussion}.

\paragraph{Convergence and Early Stopping:}
We monitor convergence using the median change in imputed values between successive iterations, as specified by the tolerance parameter $\delta$ in Algorithm~\ref{alg:weighted-imputation}. Furthermore, a maximum number of iterations is used as a safeguard.

\section{Experiments}\label{sec:Experiments}

For linear, Random Forest (RF), and Multi-Layer Perceptron (MLP) conditional model classes, we compare our method against an unweighted round-robin iterative baseline using the same conditional imputation model class, so that performance differences are primarily attributable to the inclusion of importance weights. The unweighted baseline corresponds to HyperImpute \citep{jarrett2022hyperimpute} restricted to linear conditional models in the linear setting, and to scikit-learn’s IterativeImputer \citep{scikitlearn} in the RF and MLP settings.

\subsection{Experimental Setup}\label{sec:ExperimentalSetup}

\paragraph{Datasets:}We study eight datasets spanning a range of tabular, time-series, and image modalities to evaluate performance across diverse data regimes and assess robustness of our proposed method beyond any single domain. Datasets have been preprocessed to remove zero-variance columns and to cap extreme values so that skewness lies in the range $[-10, 10]$ to ensure that results are not unduly influenced by a small number of extreme observations. All dataset details and preprocessing steps are provided in Appendix~\ref{appendix:DatasetMeta}.

\paragraph{MAR Simulation:}To simulate MAR data, we randomly select up to four continuous features and designate them as partially observed. Let $\mathcal{M} \subseteq \{1, \dots, d\}$ denote the indices of these features. For each $i \in \mathcal{M}$, the missingness mechanism is defined to depend only on a subset of fully observed covariates. Specifically, we sample a set of indices $\mathcal{C}_i \subseteq \{1, \dots, d\} \setminus \mathcal{M}$, and define the observation probability via a logistic model $
\mathbb{P}(R_i = 1 \mid X) 
= \sigma\!\left( \alpha \sum_{j \in \mathcal{C}_i} X_j + \beta_i \right)$
where $\sigma(z) = (1 + e^{-z})^{-1}$. Here, $\alpha$ controls the strength of dependence on the covariates, and $\beta_i$ is chosen to achieve a desired marginal missing rate for feature $X_i$. We repeat experiments over 35 seeds, and within each seed we vary the MAR strength parameter $\alpha \in [-3, 3] \cap \mathbb{Z}$. For each seed, the set of variables $\mathcal{M}$ and the corresponding predictor sets $\mathcal{C}_i$ are fixed. This yields $245$ simulation setups per dataset. Unless stated otherwise, $\beta_i$ is chosen to simulate a 30\% missingness rate.

\paragraph{Performance Metrics:} Let $D^{\mathrm{true}}$ denote the complete dataset and $\hat{D}$ an imputed version of an incomplete dataset $D^{\mathrm{miss}}$. We evaluate imputation accuracy using the root mean squared error (RMSE) over the missing entries, indexed by $M = \{(i,j): D^{\mathrm{miss}}_{i,j} \text{ is missing}\}$. To assess distributional similarity, we compute a marginal Wasserstein distance (WD), defined as the sum of one-dimensional Wasserstein distances across partially observed features: $\mathrm{W}(D^{\mathrm{true}}, \hat{D})
=
\sum_{j : X_j \in \mathcal{M}}
W_1\!\left(
D^{\mathrm{true}}_{(j)}, \hat{D}_{(j)}
\right)$ 
where $W_1$ denotes the 1-Wasserstein distance between empirical distributions \cite{VillaniOT}. Both metrics are calculated on the original data scale. To summarise comparative performance, for each of the 245 simulation runs we compute the performance metric under both the weighted and unweighted baselines, and report the ratio of weighted to unweighted values. The only exception is in the downstream prediction results in Figures~\ref{fig:DownstreamRegression} and~\ref{fig:DownstreamClassification}, where we also compare against an additional unimputed baseline. In this setting, we report the  performance metrics without further processing.

\subsection{Experimental Results}\label{sec:ExperimentalResults}

\begin{figure*}[ht]
  \begin{center}
    \centerline{\includegraphics[width=\textwidth]{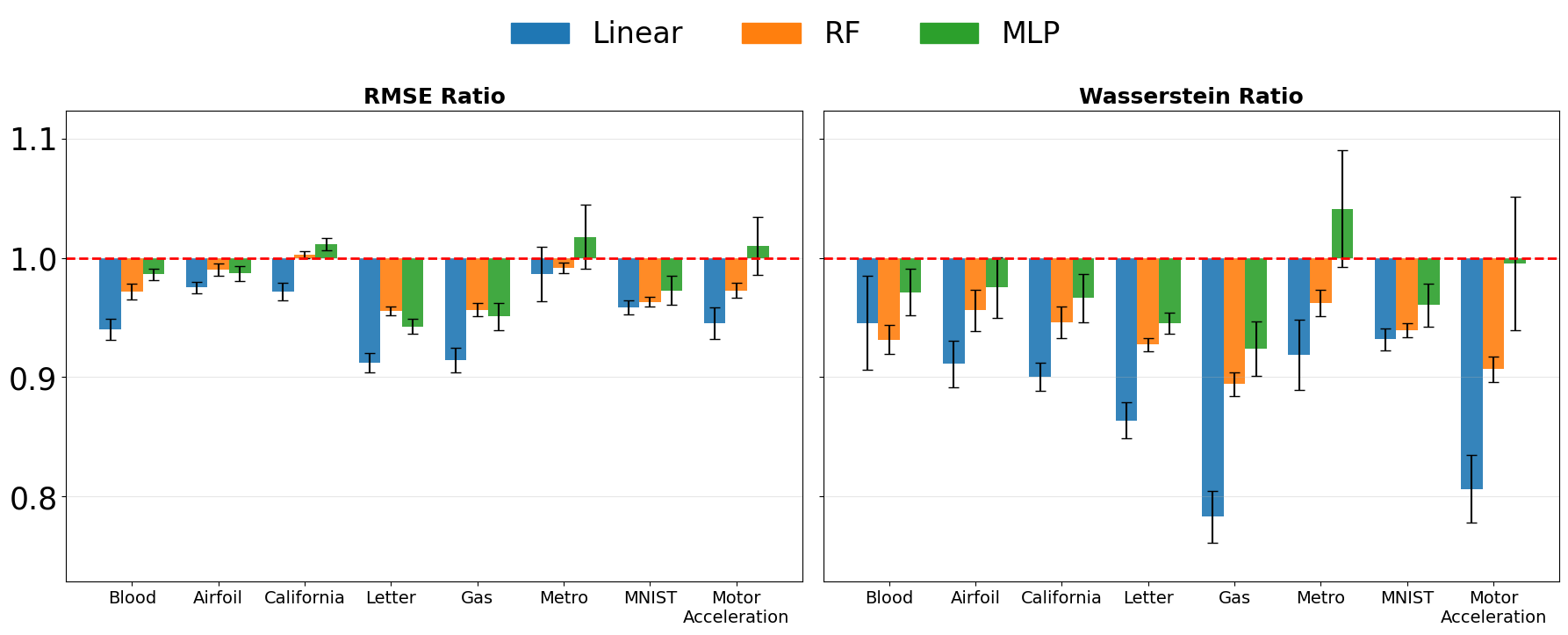}}
    \captionsetup{skip=4pt}
    \caption{
    Values below $1$ indicate improved performance. We report mean performance ratios of our weighted method relative to the unweighted baseline. Error bars represent approximate 95\% confidence intervals, computed as $\pm 1.96$ standard errors of the mean (SEM). Results are reported for RMSE and Wasserstein distance. The plot is best viewed in colour.
    }
    \label{fig:overall-iterative}
  \end{center}
\end{figure*} 

\paragraph{Overall Performance and Runtime:} Figure~\ref{fig:overall-iterative} summarises the overall results grouped by dataset and conditional imputation model class. Our method improves RMSE in $20$ out of $24$ simulation settings and Wasserstein distance in $23$ out of $24$ settings, providing strong support for the efficacy of the weighted approach. Experimental comparisons to a wider range of imputation methods are available in Appendix~\ref{appendix:ComparisonFull}. We find that RMSE gains are consistently smaller for more expressive models such as RF and MLP. This is expected behaviour, and is supported by the theory in Appendix~\ref{appendix:FurtherTheoreticalResultsCovShift}, where we show that the benefit of importance weighting diminishes as the approximation error of  conditional imputation model classes decrease. Importantly, Figure~\ref{fig:overall-iterative} shows that while improvements in RMSE are modest in some settings, with displayed 95\% confidence intervals overlapping or close to 1, Wasserstein distance exhibits consistently larger gains, with confidence intervals remaining well below 1 across most settings. These results reflect improved distributional fidelity, which \cite{vanbuuren2018flexible} identifies as important for high-quality imputation.

In terms of runtime, our proposed method achieves faster runtimes for linear conditional models and incurs only a $10$\% increase in runtime for RF and MLP models, indicating that extra computational cost remains manageable in practice. A full breakdown of run times are available in Appendix~\ref{appendix:ComparisonRuntime}, and details are provided in Appendix~\ref{appendix:gamma-discussion} on practical steps control runtime in our method.

\paragraph{Sensitivity Analysis:}
As we  simulate missingness under a known logistic regression model, the true importance weights 
$w_i = \frac{p(R_i = 0 \mid X_{\mathrm{obs}})}{p(R_i = 1 \mid X_{\mathrm{obs}})}$
can be computed for each data point. We summarise the true weights using the $95^{th}$ percentile of , $q_{95}$, where larger values indicate stronger covariate shift present in a simulation. Figure~\ref{fig:MARStrength} shows that for small $q_{95}$ --- corresponding to MCAR or weak MAR --- the weighted method provides little improvement over the unweighted baseline, consistent with the discussion in Appendix~\ref{appendix:FurtherTheoreticalResultsMCAR}. As $q_{95}$ increases, the weighted approach yields consistent improvements even under high levels of covariate shift, as evidenced by a mean ratio below 1 and 95\% confidence intervals that do not cross 1. The only exception is the MLP model, where intervals occasionally touch 1. We discuss this observation in more detail in Section~\ref{sec:Limitations}.

\begin{figure}[ht]
  \centering
  \includegraphics[width=\textwidth]{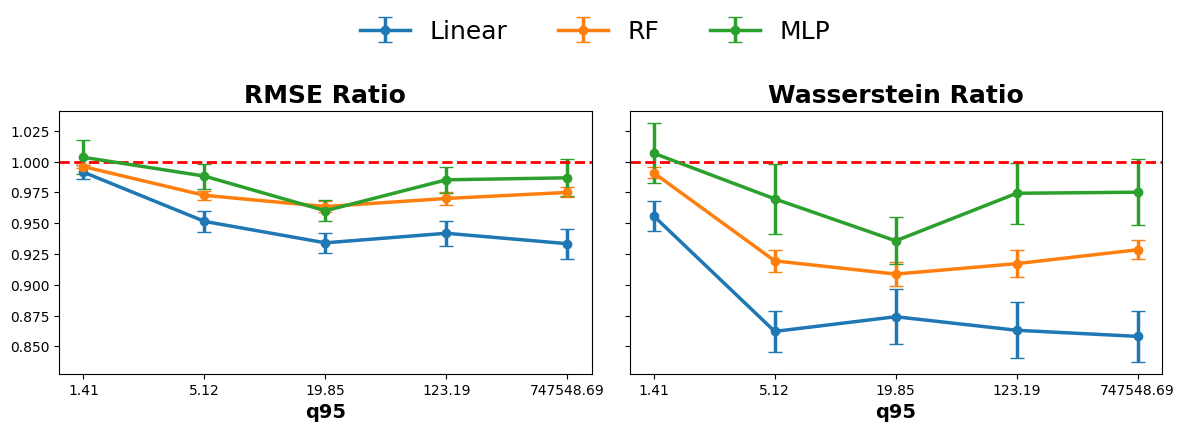}
  \captionsetup{skip=4pt}
  \caption{
 Values below 1 indicate improved performance. We report mean performance ratios of the weighted method relative to the unweighted baseline as a function of weight magnitude (summarised by the 95th percentile of the true weights). Lines correspond to different conditional imputation model types. Error bars denote $\pm 1.96$ SEM, corresponding to approximate 95\% confidence intervals. Results are reported for RMSE and Wasserstein distance. The plot is best viewed in colour.
  }
  \label{fig:MARStrength}
\end{figure}

We further study sensitivity to dataset characteristics on the Letter dataset, varying sample size, number of features, and missingness rate. Figure~\ref{fig:SensitivityAnalysis} shows that, in 55 of the 57 settings considered, the reported 95\% confidence intervals lie below or overlap with 1, indicating that the weighted method does not degrade performance and often improves upon the unweighted baseline.

\paragraph{Robustness to Model Misspecification:}
For violations of the MAR assumption, the plots in Appendix~\ref{app:RobustnessToAssumptionsMNAR} show that performance of our method degrades gradually as the probability of missingness depends increasingly on unobserved variables, with no evidence of abrupt failure. Similarly, for the weight misspecification setting, the plots in Appendix~\ref{app:RobustnessToAssumptionsWeightEstimation} show that performance degrades gradually as the missingness mechanism transitions from a linear relationship to a cubic one, with no evidence of abrupt failure. 
Full details of this study are given in Appendix~\ref{app:RobustnessToAssumptions}. 

\paragraph{Downstream Tasks:}
In Appendix~\ref{appendix:DownstreamApplication}, we evaluate imputation quality via downstream regression and classification tasks on the Gas and Letter datasets, respectively, using RMSE and accuracy as performance metrics. We compare XGBoost and linear models fitted after imputation under our weighted and unweighted approaches, along with an unimputed baseline---XGBoost without imputation; linear models with mean imputation and missingness indicators. Figures~\ref{fig:DownstreamRegression} and \ref{fig:DownstreamClassification}, together with Wilcoxon signed-rank tests, show that for regression, weighting yields statistically significant improvements for linear and MLP imputation models and typically outperforms the unimputed baseline across both XGBoost and linear predictors. For classification, XGBoost performs similarly across all imputation strategies, including the unimputed baseline. In contrast, multinomial logistic regression benefits substantially from imputation, with all methods outperforming the unimputed baseline and weighted variants showing statistically significant improvements over their unweighted counterparts. We additionally assess structural preservation by measuring how well each method recovers the empirical correlation matrix of the fully observed data. On this task, weighting substantially improves performance across all settings.

\section{Discussion}\label{sec:Discussion}
We briefly highlight connections between the techniques developed in this work and related ideas more broadly in missing data literature.  Related work has also connected missing data and covariate shift. \citet{pmlr-v206-zhou23b} study domain adaptation under missingness shift and show that the resulting problem can be interpreted as a covariate shift setting when missingness indicators are observed. While the recovered covariate shift problem is closely related to our setting, the starting assumptions and objectives differ between our two works. \citet{pmlr-v206-zhou23b} assume a pre-specified shift and focus on downstream prediction tasks, whereas we consider imputation without an explicitly specified shift and show that a distributional mismatch naturally arises from training on observed data. Similarly, \citet{10.1093/jrsssb/qkae009} frame a conformal prediction problem under covariate shift as analogous to a missing data problem, enabling the use of tools from missing data theory. This perspective is likewise developed in the context of downstream prediction tasks under an assumed shift structure and therefore differs in objective from our setting. Finally, Inverse probability weighting (IPW) methods for selection bias correction in downstream inference \citep{RobinsEtAlMARIPW, DRIPW} are also related to our work, as they adjust for discrepancies between observed and full data distributions in parameter estimation. While we also employ importance weighting to address a similar issue of distributional shift, our focus is on imputation rather than downstream parameter estimation.

\section{Limitations}\label{sec:Limitations}
While weighting improves performance in many settings, we observe situations it hinders performance. In particular, as shown in Figures~\ref{fig:overall-iterative} and~\ref{fig:MARStrength}, performance can degrade for high-variance conditional models such as MLPs, reflecting the the bias-variance trade-off associated with importance weighting \citep{Shimodaira2000CovariateShift}. Moreover, under weak-MAR regimes, our approach may offer limited improvement over existing unweighted methods. Finally, violations of the MAR assumption or poor estimation of importance weights may lead to worse performance than unweighted approaches.

\section{Conclusion}

We have shown that in risk-based approaches to imputation, where training is necessarily performed on observed data, one must account for a distributional shift between the observed and full data distributions when learning optimal imputation functions. We proposed a novel round-robin iterative imputation algorithm that explicitly accounts for this distributional shift and demonstrated the effectiveness of our approach across a range of experimental settings. Overall, our results suggest that weighted iterative imputation methods are a promising direction for further methodological development. Several avenues for future work remain. First, developing methods that are robust to weak distributional shift remains an open problem. Second, extending the framework beyond point estimation to probabilistic imputation objectives is a natural next step. Finally, our approach could be strengthened by incorporating more sophisticated density ratio estimation techniques, suggesting a promising connection to advances in this area.

\newpage
\appendix

\section{Theoretical Derivations}
\label{app:TheoreticalDerivations}

\subsection{Proof of Lemma~\ref{lemma:MSERiskUnObs}}
\label{app:MSELossDecomposes}

\begin{proof}

For this derivation, let $\mathbf{1}_{\{A\}}$ denote the indicator function of event $A$. Then, under $\mathrm{MSE}$ loss and a deterministic imputation function Definition~\ref{def:imputation-map}, our objective simplifies as follows:
\begin{align}
    \mathcal{J}(g) &= \mathbb{E}_{X, R}\Big[
\text{MSE}\big\{ g(\tilde{X}), X\big\}\Big]\nonumber \\
    &= \mathbb{E}_{X, R}\left[\sum_{i=1}^d
\big\{g_i(\tilde{X}) - X_i\big\}^2\right]\label{eq:MSELossStep0} \\
    &= \mathbb{E}_{X, R}\left[\sum_{i=1}^d \left\{R_i X_i + (1-R_i)g_i(\tilde{X}) - X_i\right\}^2\right] \label{eq:MSELossStep1} \\
    &= \mathbb{E}_{X, R}\left[\sum_{i=1}^d \{g_i(\tilde{X}) - X_i\}^2 \mathbf{1}_{\{R_i=0\}}\right]
    \label{eq:MSELossStep2} \\
    &= \sum_{i=1}^d \mathbb{E}_{X, R}\left[\{g_i(\tilde{X}) - X_i\}^2 \mathbf{1}_{\{R_i=0\}}\right] \label{eq:MSELossStep3}
\end{align}

Equation~\eqref{eq:MSELossStep0} follows by definition of MSE.  Equation~\eqref{eq:MSELossStep1} follows by writing $g_i(\tilde{X})$ as an $R_i$-weighted sum of the true value $X_i$ and the imputed value $g_i(\tilde{X})$ which holds due to the assumed property that imputation functions preserve already observed coordinates of $\tilde{X}$. Equation~\eqref{eq:MSELossStep2} by noting that only coordinates with $R_i = 0$ contribute to the loss, and Equation~\eqref{eq:MSELossStep3} by linearity of expectation.

Focusing on the $i^{\text{th}}$ term of the sum, we have:
\begin{align}
\mathbb{E}_{X, R}\left[\{g_i(\tilde{X}) - X_i\}^2 \mathbf{1}_{\{R_i=0\}}\right]
&= \mathbb{E}_{R_i}\left(\mathbb{E}_{X, R_{\neg i} \mid R_i}\left[\left\{g_i(\tilde{X}) - X_i\right\}^2 \mathbf{1}_{\{R_i=0\}}\right]\right)\label{eq:MSELosstermistep1} \\ 
&= \mathbb{E}_{R_i}\left(\mathbf{1}_{\{R_i=0\}}\mathbb{E}_{X, R_{\neg i} \mid R_i}\left[\left\{g_i(\tilde{X}) - X_i\right\}^2\right]\right)\label{eq:MSELosstermistep2} \\
&= \sum_{r_i \in \{0, 1\}}\mathbf{1}_{\{r_i=0\}}\mathbb{E}_{X, R_{\neg i} \mid R_i = r_i}\left[\left\{g_i(\tilde{X}) - X_i\right\}^2\right]p(R_i = r_i)\label{eq:MSELosstermistep3}\\
&= \mathbb{E}_{X, R_{\neg i} \mid R_i = 0}\left[\left\{g_i(\tilde{X}) - X_i\right\}^2\right] \, p(R_i=0)\label{eq:MSELosstermistep4}
\end{align}

Equation~\eqref{eq:MSELosstermistep1} follows by the law of total expectation. Equation~\eqref{eq:MSELosstermistep2} follows by pulling $\mathbf{1}_{\{R_i=0\}}$ outside the inner expectation when conditioning on $R_i$. Equation~\eqref{eq:MSELosstermistep3} follows by the definition of expectation over $R_i$. Equation~\eqref{eq:MSELosstermistep4} follows because the $\mathbf{1}_{\{r_i=0\}}$ kills the term in the sum with $r_i = 1$.

Therefore, 
\begin{equation}
    \mathcal{J}(g) = \mathbb{E}_{X, R}\Big[
\text{MSE}\big\{(X_{\mathrm{obs}}, g(\tilde{X})), X\big\}\Big] = \sum_{i=1}^dp(R_i = 0)\mathbb{E}_{X, R_{\neg i}}\left[\left\{g_i(\tilde{X}) - X_i\right\}^2 \mid R_i = 0\right]\nonumber
\end{equation}

Finally, $\forall i \not \in \mathcal{P}$, $p(R_i = 0) = 0$ and hence these terms drop out of the sum yielding 
\begin{equation}
    \mathcal{J}(g) = \sum_{i \in \mathcal{P}}p(R_i = 0)\mathbb{E}_{X, R_{\neg i}}\left[\left\{g_i(\tilde{X}) - X_i\right\}^2 \mid R_i = 0\right]\nonumber
\end{equation}
\end{proof}

\newpage

\subsection{Proof of Proposition~\ref{prop:GeneralImputationObs}}
\label{app:rewritingMSEobjective}

\begin{proof}

We aim to rewrite the $i^{th}$- coordinate risk
\[
p(R_i = 0)\mathbb{E}_{X, R_{\neg i}}\left[\left\{g_i(\tilde{X}) - X_i\right\}^2 \mid R_i = 0\right]
\]
in terms of observed data only. Writing $w_i(x_{\mathrm{obs}}, r_{\neg i}) := \frac{p(x_{\mathrm{obs}}, r_{\neg i} \mid R_i = 0)}{p(x_{\mathrm{obs}}, r_{\neg i} \mid R_i = 1)}$ and ignoring the $p(R_i = 0)$ term for now, we obtain the following,
\begin{align} &\mathbb{E}_{X, R_{\neg i} \mid R_i = 0}\Big[\{g_i(\tilde{X}) - X_i\}^2\Big]\nonumber\\
&= \mathbb{E}_{X_{\mathrm{obs}}, R_{\neg i} \mid R_i = 0} 
\Big(\, \mathbb{E}_{X_{\mathrm{miss}} \mid X_{\mathrm{obs}}, R_{\neg i}, R_i = 0}\Big[ \{g_i(\tilde{X}) - X_i\}^2\Big] \, \Big)\label{MSERiskObsStep0} \\
&= \mathbb{E}_{X_{\mathrm{obs}}, R_{\neg i} \mid R_i = 0} 
\Big(\, \mathbb{E}_{X_{\mathrm{miss}} \mid X_{\mathrm{obs}}, R_{\neg i}, R_i = 1}\Big[ \{g_i(\tilde{X}) - X_i\}^2\Big] \, \Big) \label{MSERiskObsStep1} \\
&= \int \mathbb{E}_{X_{\mathrm{miss}} \mid X_{\mathrm{obs}}, R_{\neg i}, R_i = 1}\Big[ \{g_i(\tilde{X}) - X_i\}^2\Big] 
\, p(x_{\mathrm{obs}}, r_{\neg i} \mid R_i = 0) \, d x_{\mathrm{obs}} \, d r_{\neg i}\label{MSERiskObsStep2} \\
&= \int \mathbb{E}_{X_{\mathrm{miss}} \mid X_{\mathrm{obs}}, R_{\neg i}, R_i = 1}\Big[ \{g_i(\tilde{X}) - X_i\}^2\Big] 
\, p(x_{\mathrm{obs}}, r_{\neg i} \mid R_i = 1) \, w_i(x_{\mathrm{obs}}, r_{\neg i})
\, dx_{\mathrm{obs}} \, dr_{\neg i}\label{MSERiskObsStep3} 
\\
&= \mathbb{E}_{X_{\mathrm{obs}}, R_{\neg i} \mid R_i = 1} 
\Big(\, \mathbb{E}_{X_{\mathrm{miss}} \mid X_{\mathrm{obs}}, R_{\neg i}, R_i = 1}\Big[ w_i(x_{\mathrm{obs}}, r_{\neg i})\{g_i(\tilde{X}) - X_i\}^2\Big] \, \Big) \label{MSERiskObsStep4} \\
&= \mathbb{E}_{X_{\mathrm{obs}}, R_{\neg i} \mid R_i = 1} \Big[
w_i(X_{\mathrm{obs}}, R_{\neg i})
(g_i(\tilde{X}) - X_i)^2 
\Big]\label{MSERiskObsStep5} 
\end{align}

Equation~\eqref{MSERiskObsStep0} follows by the law of total expectation. Equation~\eqref{MSERiskObsStep1} follows from the MAR assumption 
$R \perp\!\!\!\perp X_{\mathrm{miss}} \mid X_{\mathrm{obs}}$, which implies that the conditional distribution of $X_{\mathrm{miss}}$ does not depend on $R_i$, 
and we may replace $R_i=0$ by $R_i=1$ without altering the inner expectation. Equation~\eqref{MSERiskObsStep2} follows by writing out explicitly the outer expectation. Equation~\eqref{MSERiskObsStep3} follows by factoring $p(x_{\mathrm{obs}}, r_{\neg i} \mid R_i = 0) = p(x_{\mathrm{obs}}, r_{\neg i} \mid R_i = 1) \, w_i(x_{\mathrm{obs}}, r_{\neg i})$. 
Equation~\eqref{MSERiskObsStep4} follows by bringing the weighting inside the inner expectation and recognising that the outer expectation is now over $(X_{\mathrm{obs}}, R_{\neg i}) \mid R_i = 1$. Equation~\eqref{MSERiskObsStep5} follows by iterated expectation.

It remains to verify that the weights are well-defined. Points with 
$p(X_{\mathrm{obs}}, R_{\neg i} \mid R_i = 0) = 0$ do not contribute to the expectation and can be ignored. 
For points where $p(X_{\mathrm{obs}}, R_{\neg i} \mid R_i = 0) > 0$, Bayes' rule gives
\[
p(X_{\mathrm{obs}}, R_{\neg i} \mid R_i = 1)
= \frac{p(R_i = 1 \mid X_{\mathrm{obs}}, R_{\neg i})\, p(X_{\mathrm{obs}}, R_{\neg i})}{p(R_i = 1)} > 0
\]
by the positivity assumption. Hence,
\[
p(X_{\mathrm{obs}}, R_{\neg i} \mid R_i = 1) > 0
\quad \text{whenever} \quad
p(X_{\mathrm{obs}}, R_{\neg i} \mid R_i = 0) > 0,
\]
and the weights $w_i(X_{\mathrm{obs}}, R_{\neg i})$ are well-defined.

\end{proof}
\begin{remark}
Equation~\eqref{MSERiskObsStep1} highlights the source of covariate shift induced by the MAR mechanism. In practice, we may train a model to predict $X_i$ by minimising squared error over samples for which $R_i = 1$. This corresponds to the inner expectation. However, the quantity of interest is the performance of our learned model on the population where $R_i = 0$, corresponding to the outer expectation. In general, the distributions of $(X_{\mathrm{obs}}, R_{\neg i}) \mid R_i = 1$ and $(X_{\mathrm{obs}}, R_{\neg i}) \mid R_i = 0$ differ, leading to a covariate shift between the training and target populations.

If this shift is not accounted for, the resulting estimator may be biased. The weighting function $w_i(X_{\mathrm{obs}}, R_{\neg i})$ corrects for this discrepancy by re-weighting the observed-data distribution to match the target distribution.
\end{remark}

\newpage

\section{Further theoretical results}\label{appendix:FurtherTheoreticalResults}

In this section, we derive theoretical results that help explain the empirical patterns observed in Section~\ref{sec:Experiments}. In particular, we investigate when the optimal solutions learned under the subpopulations $R_i = 0$ and $R_i = 1$ coincide. We consider risk minimisation for coordinate $i$ for which we are choosing an optimal imputation function from the model class $\mathcal{G}_i$.

For this section, let:
\begin{itemize}
	\item 
    $\mathcal{J}_i^{(r)}(g_i) = \mathbb{E}\Big[\big(g_i(\tilde{X}) - X_i\big)^2 \,\big|\, R_i = r\Big], \quad r \in \{0,1\}$, be the optimisation problems corresponding to coordinate $i$ under $R_i = r$
	\item $g_i^{*, (r)} = \underset{g \in \mathcal{G}_i}{\arg \min} \left\{ \mathcal{J}_i^{(r)}(g) \right\}$ be the loss minimising imputation function for column $i$ chosen from a class $\mathcal{G}_i$ under $R_i = r$
	\item $f_i(\tilde{X}) = \mathbb{E}[X_i \mid \tilde{X}] = \mathbb{E}[X_i \mid X_{\mathrm{obs}}]$ be the theoretically optimal imputation function. We note that, under MAR, this solution is optimal under $R_i = 1$ and $R_i = 0$ and, in actuality, depends only on $X_{\mathrm{obs}}$.
	\item $\varepsilon_{i, (r)}^2 := \underset{g_i \in \mathcal{G}_i}{\min}
    \mathbb{E}\Big[\big(g_i(\tilde{X}) - f_i(\tilde{X})\big)^2 \,\big|\, R_i = r\Big]
$ be the approximation error of $\mathcal{G}_i$ under $R_i=r$
\end{itemize}

In the above definitions $\mathcal{J}_i^{(0)}(g_i)$ is the true coordinate-wise risk of coordinate $g_i$ as defined in \ref{eq:MSERiskUnobsComponent}, whereas $\mathcal{J}_i^{(1)}(g_i)$ represents the unweighted objective minimised by existing unweighted imputation algorithms.

\subsubsection{Agreement of population-optimal imputation functions under covariate shift}\label{appendix:FurtherTheoreticalResultsCovShift}
 
\begin{proposition}[Distance between population minimisers under covariate shift]\label{prop:ConvergingSolutions}
	Suppose that:
	\begin{itemize}
		\item $\exists \kappa > 0$ such that $\frac{dP(\tilde X \mid R_i=0)}{dP(\tilde X \mid R_i=1)} \le \kappa^{-1}$ almost surely (overlap condition).
        \item All functions $g_i \in \mathcal{G}_i$ belong to the intersection
        \[
        g_i \in L^2(\tilde X \mid R_i = 0) \cap L^2(\tilde X \mid R_i = 1),
        \]
        i.e., they have finite second moments under both conditional distributions $P(\tilde X \mid R_i = 0)$ and $P(\tilde X \mid R_i = 1)$. This ensures that all terms in the proof are well-defined and that we can apply results such as the triangle inequality in $L^2$ and have an intepretation of distance.
	\end{itemize}
	Then the population minimisers satisfy
	\[
	\mathbb{E}\Big[ \big(g_i^{*,(0)}(\tilde{X}) - g_i^{*,(1)}(\tilde{X})\big)^2 \,\big|\, R_i = 0 \Big] 
	\le \big(\varepsilon_{i, (0)} + \kappa^{-1/2}\varepsilon_{i, (1)}\big)^2.
	\]
\end{proposition}

\begin{proof}
For any $g_i \in \mathcal{G}_i$ and $r \in \{0, 1\}$, 
by the bias-variance decomposition gives 
\begin{equation}\label{eq:BiasVarianceDecomp}
\mathbb{E}\big[(g_i(\tilde X) - X_i)^2 \mid R_i = r\big] 
= \mathbb{E}\big[(g_i(\tilde X) - f(\tilde X))^2 \mid R_i = r\big] + \mathbb{E}\big[(X_i - f(\tilde X))^2 \mid R_i = r\big]
\end{equation}

Since $\mathbb{E}\big[(X_i - f(\tilde X))^2 \mid R_i = r\big]$ does not depend on $g_i$, the minimiser, $g_i^{*,(r)}$, of the left hand side of \ref{eq:BiasVarianceDecomp} must also be the minimiser of $\mathbb{E}\big[(g_i(\tilde X) - f(\tilde X))^2 \mid R_i = r\big]$. Therefore, 
\[
\mathbb{E}\big[(g_i^{*,(r)}(\tilde X) - f(\tilde X))^2 \mid R_i = r\big] = \varepsilon_{i, (r)}^2
\]

Now, consider the $L^2$ distance between the optimal functions under $R_i = 0$ and $R_i = 1$:
\begin{align}
&\mathbb{E}\big[ (g_i^{*,(0)}(\tilde X) - g_i^{*,(1)}(\tilde X))^2 \mid R_i = 0 \big]^{1/2}\nonumber\\
&\le \mathbb{E}\big[ (g_i^{*,(0)}(\tilde X) - f_i(\tilde X))^2 \mid R_i = 0 \big]^{1/2} 
+ \mathbb{E}\big[ (g_i^{*,(1)}(\tilde X) - f_i(\tilde X))^2 \mid R_i = 0 \big]^{1/2}\label{eq:TriangleInequality}
\end{align}
by the triangle inequality in $L^2$.  

The first term is exactly $\varepsilon_{i, (0)}$. For the second term, the triangle inequality gives us:
\begin{align}
\mathbb{E}\big[ (g_i^{*,(1)}(\tilde X) - f(\tilde X))^2 \mid R_i = 0 \big] 
&= \int (g^{*,(1)}(\tilde x) - f(\tilde x))^2 \, dP(\tilde x \mid R_i=0)\nonumber\\
&\le \kappa^{-1} \int (g^{*,(1)}(\tilde x) - f(\tilde x))^2 \, dP(\tilde x \mid R_i=1)
= \kappa^{-1} \varepsilon_{i, (1)}^2\label{eq:UpperBound}
\end{align}

Combining \eqref{eq:TriangleInequality} and \eqref{eq:UpperBound}, and squaring both sides gives the stated bound:
\[
\mathbb{E}\big[ (g^{*,(0)}(\tilde X) - g^{*,(1)}(\tilde X))^2 \mid R_i = 0 \big] \le \big(\varepsilon_{i, (0)} + \kappa^{-1/2}\varepsilon_{i, (1)}\big)^2.
\]
\end{proof}

\begin{corollary}\label{corr:ConvergingSolutions}[Optimisation problems are equivalent for well specified models]
	Suppose that $f_i \in \mathcal{G}_i$
	Then $g^{*,(0)}(\tilde{X}) = g^{*,(1)}(\tilde{X})$
\end{corollary}

\begin{proof}
   If $f_i \in \mathcal{G}_i$, then $\varepsilon_0 = \varepsilon_1 = 0$. The result follows directly from Proposition~\ref{prop:ConvergingSolutions}.

\end{proof}

\paragraph{Remark:}
Proposition~\ref{prop:ConvergingSolutions} shows that as the expressive power of the model class $\mathcal{G}_i$ increases, and hence the approximation errors $\varepsilon_0$ and $\varepsilon_1$ decrease, the population-optimal solutions under $R_i = 0$ and $R_i = 1$ converge. This provides a theoretical explanation for the empirical observation that our method yields larger improvements for simpler model classes, such as linear conditional models, while the gains diminish for more expressive models such as MLPs and random forests, where approximation error is already small.

\subsubsection{Agreement of population-optimal imputation functions under MCAR missingness}\label{appendix:FurtherTheoreticalResultsMCAR}

\begin{proposition}
[Equivalence of optimisation problems under MCAR]
\label{prop:MCARIndependentR}
Suppose that the data satisfy the MCAR assumption. Then,
\[
g_i^{*,(0)}(\tilde{X}) = g_i^{*,(1)}(\tilde{X}).
\]
\end{proposition}

\begin{proof}
As MCAR is a special case of MAR, our derived optimal weights in \eqref{eq:importanceWeights} are still valid. As argued in section~\ref{sec:SurrogateRisk}, $R_{\neg i}$ can be ignored in this weighting yielding
\[
w_i(X_{\mathrm{obs}})
=
\frac{p(X_{\mathrm{obs}} \mid R_i = 0)}{p(X_{\mathrm{obs}} \mid R_i = 1)}.
\]
Under the MCAR assumption, the distribution of $X_{\mathrm{obs}}$ does not depend on $R_i$, and hence
\[
p(X_{\mathrm{obs}} \mid R_i = 0) = p(X_{\mathrm{obs}} \mid R_i = 1),
\]
so that $w_i(X_{\mathrm{obs}}) \equiv 1$ almost surely.

Consequently, the objectives $\mathcal{J}_i^{(0)}(g)$ and $\mathcal{J}_i^{(1)}(g)$ coincide, and therefore their population minimisers are identical.
\end{proof}

\newpage

\section{Supplementary Experimental Results}\label{appendix:ExperimentsSupplimentary}

\subsection{Dataset Metadata and Preprocessing}\label{appendix:DatasetMeta}

\begin{table}[ht]
\centering
\caption{Summary of datasets used in experiments with dataset URLs}
\label{tab:dataset_summary}
\begin{tabular}{l c c l p{6cm}}
\hline
\textbf{Dataset} 
& \textbf{Rows} 
& \textbf{Columns} 
& \textbf{Data type} 
& \textbf{URL} \\
\hline
Blood 
& 748
& 3 
& Tabular 
& \url{https://doi.org/10.24432/C5GS39.} \\
Airfoil 
& 1503 
& 6 
& Tabular 
& \url{https://doi.org/10.24432/C5VW2C} \\
California 
& 20640
& 8 
& Tabular 
& \url{https://scikit-learn.org/stable/modules/generated/sklearn.datasets.fetch_california_housing.html} \\
Letter 
& 20000
& 16
& Tabular 
& \url{https://doi.org/10.24432/C5ZP40} \\
Gas 
& 416153 
& 9 
& Time series 
& \url{https://doi.org/10.24432/C5762W} \\
Metro 
& 1516948
& 15
& Time series 
& \url{https://doi.org/10.24432/C5VW3R} \\
Motor Acceleration 
& 511806
& 57
& Time series 
& \url{https://doi.org/10.18419/DARUS-3301} \\
MNIST 
& 70000
& 196
& Image 
& \url{http://yann.lecun.com/exdb/mnist/} \\
\hline
\end{tabular}
\end{table}

Table \ref{tab:dataset_summary} gives a brief summary of the datasets used in our study. Some preprocessing was applied to datasets where necessary.

\begin{itemize}
    \item \textbf{Blood:} Remove one perfectly co-linear column.
    \item \textbf{California:} Capped columns so that skewness lies in the range $[-10, 10]$
    \item \textbf{Motor Acceleration:} Removed unary columns and capped columns so that skewness lies in the range $[-10, 10]$
    \item \textbf{MNIST:} To reduce computational overhead while preserving the salient digit structure, images were downsampled from $28 \times 28$ to $14 \times 14$ for the experiments.

\end{itemize}

\subsection{Non-Iterative Imputation Method Comparison}\label{appendix:ComparisonFull}

Beyond the experimental results provided in the Section~\ref{sec:ExperimentalResults}, here we provide a comparison of our algorithm against a wider range of state-of-the-art imputation methods. In particular we consider \begin{itemize}
    \item \textbf{HyperImpute}: An iterative imputation method that performs a principled model search at each iteration to choose the best performing conditional model. \citep{jarrett2022hyperimpute}
    \item \textbf{GAIN}: A GAN-based approach in which a generator imputes missing entries and a discriminator distinguishes them from observed entries. \citep{yoon2018gainmissingdataimputation}
    \item \textbf{Sinkhorn}: An approach that minimises an optimal transport divergence between distributions of observed and imputed entries. \citep{MuzellecEtAllSinkhorn}
    \item \textbf{SoftImpute}: An approach that performs low-rank matrix completion using nuclear-norm regularization with iterative soft-thresholded SVD.\citep{MazumderEtAlSoftimpute}
\end{itemize}  

The results of this wider study are shown in Figure~\ref{fig:OverallComparison}.

\begin{figure*}[ht]
    \centering
    \includegraphics[width=\textwidth]{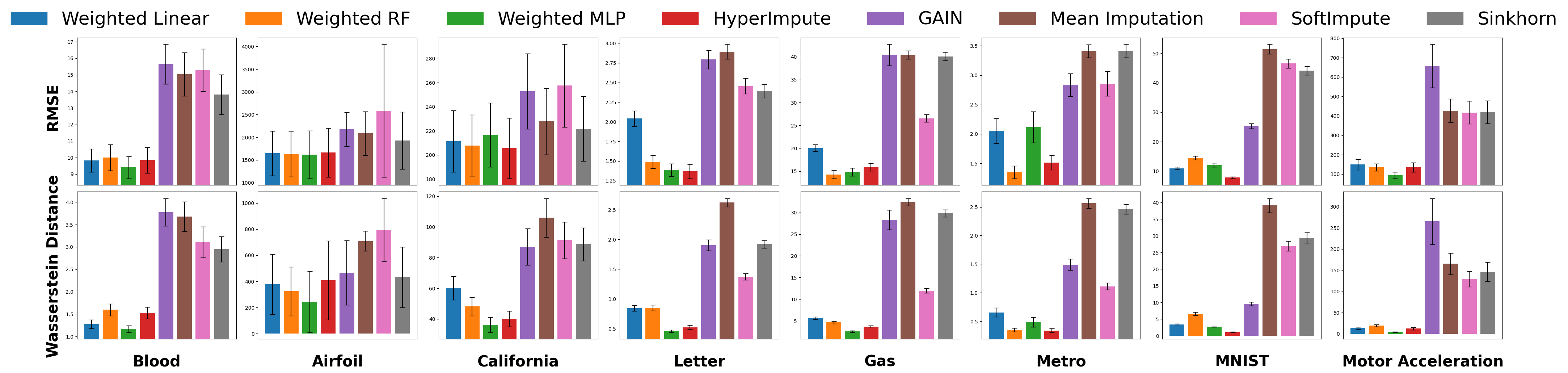}
    \captionsetup{skip=4pt}
    \caption{
        Lower values indicate better performance. The figure compares a broad range of methods, showing that the weighted variants are competitive across all scenarios. Bars denote mean performance, with error bars given by $\pm 1.96$ SEM, corresponding to approximate 95\% confidence intervals. Results are reported for RMSE and Wasserstein distance. The plot is best viewed in colour.}    \label{fig:OverallComparison}
\end{figure*}

\subsection{Runtimes}\label{appendix:ComparisonRuntime}

Table~\ref{tab:runtime_full} summarises the runtimes for our experiments. 

Across all experiments, we consistently observe improved runtime efficiency for linear models when using our weighted approach. When aggregated across all datasets, the weighted linear method runs in approximately 47\% of the time of the corresponding unweighted approach, demonstrating a substantial computational advantage. We note that HyperImpute runtimes are not directly comparable to the our approach under linear conditional models as HyperImpute performs hyperparameter search at each iteration, incurring overhead beyond base model fitting that accounts for its longer runtimes.

For both RF and MLP based models, we observe relatively larger proportional increases in runtime on smaller datasets. However, these differences are small in absolute terms, typically amounting to only a few seconds, and therefore have limited practical impact.

Importantly, as dataset size increases (from the Gas dataset onwards), runtimes between weighted and unweighted approaches become more comparable, showing our proposed method scales well without introducing prohibitive computational overhead.

At the overall level, the weighted variants are approximately \textbf{12.5\% slower for RF models} and \textbf{11.3\% slower for MLP models}. Despite this modest overhead, these results demonstrate that a weighted iterative imputation approach can be deployed in practice without requiring expensive additional estimation or tuning of weighting parameters, and without introducing significant computational burden.

\begin{table}[ht]
\centering
\caption{Mean runtime (seconds) across datasets with overall mean highlighted in bold.}
\label{tab:runtime_full}
\small
\setlength{\tabcolsep}{3pt}
\renewcommand{\arraystretch}{1.2}

\begin{tabular}{llccccccccc}
\hline
 &  & Blood & Airfoil & California & Letter & Gas & Metro & MNIST & Motor Acc. & \textbf{Overall} \\
\hline

\multirow{2}{*}{Linear}
& Weighted
& 1.26 & 0.25 & 0.42 & 0.69 & 2.28 & 11.91 & 42.76 & 17.27 & \textbf{9.11} \\
& Unweighted
& 4.70 & 9.67 & 1.53 & 1.73 & 5.16 & 27.23 & 57.25 & 51.08 & \textbf{19.30} \\
\hline

\multirow{2}{*}{RF}
& Weighted
& 4.53 & 14.85 & 25.23 & 21.22 & 112.26 & 385.59 & 224.17 & 891.74 & \textbf{210.45} \\
& Unweighted
& 0.29 & 2.13 & 7.69 & 5.03 & 93.26 & 415.20 & 51.79 & 926.87 & \textbf{187.00} \\
\hline

\multirow{2}{*}{MLP}
& Weighted
& 3.19 & 15.66 & 74.99 & 90.31 & 826.31 & 1603.54 & 344.30 & 852.64 & \textbf{476.37} \\
& Unweighted
& 0.28 & 4.22 & 60.05 & 69.27 & 815.18 & 1581.80 & 203.27 & 792.08 & \textbf{428.02} \\
\hline

\end{tabular}
\end{table}

\subsection{Gamma Tuning}\label{appendix:gamma-discussion}

A naive implementation of our algorithm requires re-estimating weights and tuning the tempering parameter $\gamma$ at each iteration, via grid search with cross-validation. This can lead to a substantial increase in computational cost on top of the already iterative nature of the imputation procedure. Careful handling of $\gamma$ is therefore essential to ensure tractable runtime.

We propose a simple guiding principle: \textbf{as the expressivity of the conditional model increases, $\gamma$ can be tuned over a coarser grid and updated less frequently}. This is supported by the theoretical results in Appendix~\ref{appendix:FurtherTheoreticalResultsCovShift}, which show that the impact of importance weighting diminishes as model expressivity increases. Consequently, less computational effort needs to be spent on precisely tuning $\gamma$ in these settings.

In practice, we adopt a two-stage strategy: (i) an initial grid search to select $\gamma$, followed by (ii) local refinement via warm-starting in subsequent iterations.

\paragraph{Initial grid search:}
We select the initial $\gamma$ over a model-dependent grid reflecting the expressivity of the conditional model. Less expressive models require finer tuning, while more expressive models permit coarser grids.

\paragraph{Warm starting:}
After the initial selection, $\gamma$ is updated, in successive iterations, within a small neighbourhood of the previous value. Specifically, at iteration $t$, we search over
\[
\{\gamma^{(t-1)} - 0.05,\ \gamma^{(t-1)},\ \gamma^{(t-1)} + 0.05\},
\]

where $\gamma^{(t-1)}$ is the gamma value previously selected for our column of interest. This approach significantly reduces computational cost while ensuring stable convergence.

\paragraph{Weight update frequency:} We are not required to update weights at every iteration.
For linear models, weights and $\gamma$ are updated at every iteration. For more expressive models (RF and MLP), we reduce update frequency: an initial unweighted pass is performed, followed by two iterations where weights estimates are updated and $\gamma$ is tuned, after which both are fixed for the remaining iterations.

The specific configurations used in our experiments are summarised in Table~\ref{tab:gamma-tuning}.

\begin{table}[h]
\centering
\caption{Summary of $\gamma$ tuning strategy across model classes.}
\label{tab:gamma-tuning}
\begin{tabular}{lccc}
\toprule
Model & Initial $\gamma$ Grid & Warm Start & Update Strategy \\
\midrule
Linear & $\{0, 0.1, \dots, 1\}$ & $\pm 0.05$ & Every iteration \\
RF & $\{0, 0.5, 1\}$ & $\pm 0.05$ & Iterations 2--3, then fixed \\
MLP & $\{0, 0.25, 0.5, 0.75, 1\}$ & $\pm 0.05$ & Iterations 2--3, then fixed \\
\bottomrule
\end{tabular}
\end{table}

\subsection{Sensitivity Analysis}\label{appendix:SensitivityAnalysis}

To assess the sensitivity of our method to different data characteristics, we evaluate performance under varying data sizes, feature dimensionality, and missingness rates. Data size is controlled by subsampling rows from the original dataset, while feature count is varied by subsampling columns. Missingness is adjusted by tuning $\beta_i$ to achieve the desired missingness level as described in Section~\ref{sec:Experiments}. Figure~\ref{fig:SensitivityAnalysis} shows that our method remains robust across the range of scenarios considered, consistently outperforming the unweighted baseline.

\begin{figure*}[ht]
  \centering
  \includegraphics[width=\textwidth]{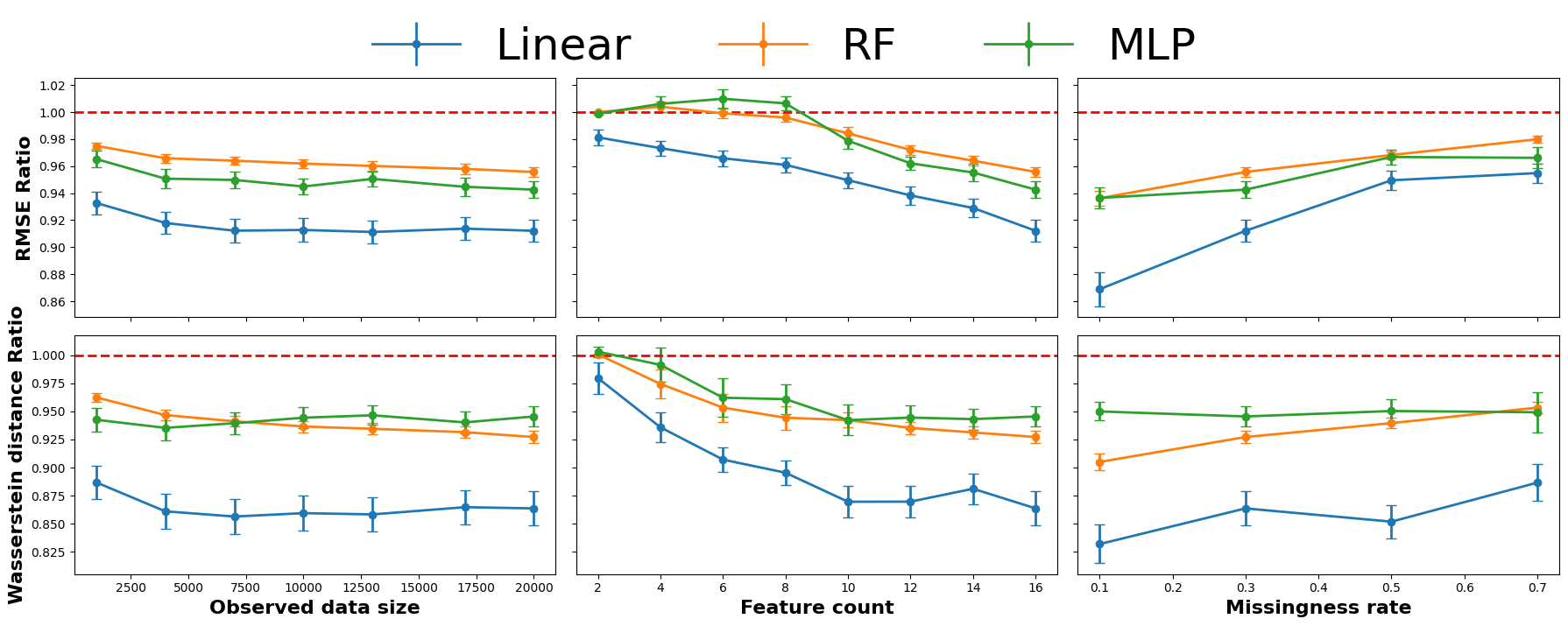}
    \captionsetup{skip=4pt}
    \caption{Values below 1 indicate improved performance. We report  mean performance ratios of the weighted method relative to its unweighted counterpart across varying dataset conditions and conditional model types. Lines denote mean ratios across dataset settings, with error bars given by $\pm 1.96$ SEM, corresponding to approximate 95\% confidence intervals. Results are reported for RMSE and Wasserstein distance across data variation regimes (e.g., sample size, number of features, and missingness rate). The plot is best viewed in colour.}
    \label{fig:SensitivityAnalysis}
\end{figure*}

\subsection{Robustness to Model and Mechanism Misspecification}\label{app:RobustnessToAssumptions}

In this section, we investigate how performance degrades under progressively stronger violations of the MAR assumption and under misspecification of the weighting model.

We consider data generated according to the following process:
\begin{align*}
X_1 &\sim \mathcal{N}(0, 1), \\
X_2 &\sim \mathcal{N}(0, 1), \\
X_3 &= X_1 X_2 + \epsilon_3, \quad \epsilon_3 \sim \mathcal{N}(0, 1), \\
X_4 &= 0.5 X_1^2 + X_2 + \epsilon_4, \quad \epsilon_4 \sim \mathcal{N}(0, 1), \\
X_5 &= X_1 X_3 + X_2^2 + \epsilon_5, \quad \epsilon_5 \sim \mathcal{N}(0, 1).
\end{align*}

\subsubsection{MNAR}\label{app:RobustnessToAssumptionsMNAR}

To simulate MNAR mechanisms, we introduce missingness in $X_4$ and $X_5$ according to
\[
P(R_j = 1 \mid X) = \sigma\!\left(c + \beta_1 X_1 + \beta_2 X_2 + \beta_3 X_3 \right), \quad j \in \{4,5\},
\]
where $\sigma(\cdot)$ is the logistic function and $c$ is calibrated to control the overall missingness rate.

\paragraph{Simulation setup 1 (missingness driven by $X_1$):}
We fix $\beta_2 = \beta_3 = 2$ and vary $\beta_1$ over the range $[0, 10]$. After generating missingness, the driver variable $X_1$ is removed from the partially observed dataset. This yields an MNAR setting in which missingness depends on an unobserved variable, with the strength of this dependence controlled by $\beta_1$.

Figure~\ref{fig:MNARX1} reports the ratio of weighted to unweighted performance for linear conditional imputation models. When $\beta_1 = 0$, the mechanism reduces to MAR, as missingness depends only on observed variables. As $\beta_1$ increases, the degree of MNAR violation increases.

The vertical line indicates the value $\beta_1^\star$ at which the majority of variance in the missingness mechanism is explained by the unobserved variable $X_1$, rather than the observed covariates. As $\beta_1$ increases, we observe a smooth degradation in performance for RMSE, while improvements in Wasserstein distance remain relatively stable across a wide range of MNAR settings.

We note that large values such as $\beta_1 = 10$ are included for illustrative purposes to probe extreme regimes. Overall, these results suggest that our proposed approach continues to offer improvement provided that the primary sources of missingness dependence are captured within the observed data.

\begin{figure}[ht]

    \centering
    \includegraphics[width=\textwidth]{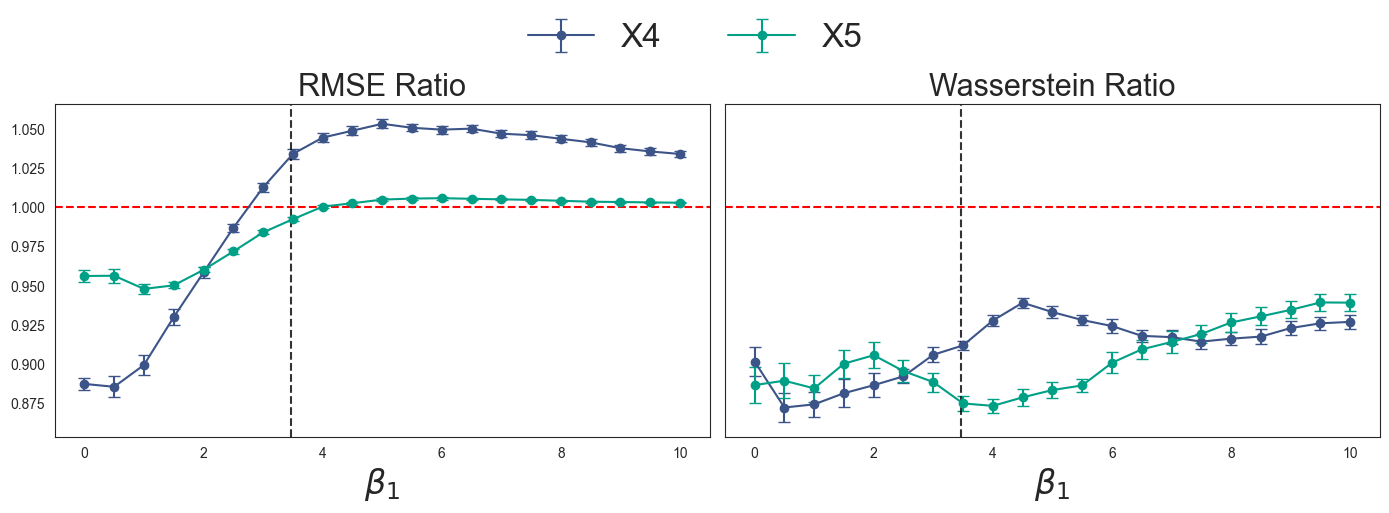}
    \captionsetup{skip=4pt}
    \caption{Values below 1 indicate improved performance. Mean performance ratios of the weighted method relative to its unweighted counterpart are shown for linear conditional models. Lines denote mean ratios, with error bars given by $\pm 1.96$ SEM, corresponding to approximate 95\% confidence intervals. The vertical line indicates the value $\beta_1^\star$, at which the majority of variance in the missingness mechanism is explained by the unobserved variable $X_1$ rather than the observed covariates. The plot is best viewed in colour.}\label{fig:MNARX1}
\end{figure}

\paragraph{Simulation setups 2 \& 3 (missingness driven by $X_2$ and $X_3$):}
We repeat the above procedure for two additional settings. In the first, missingness is driven by $X_2$, after which $X_2$ is removed from the dataset prior to imputation. In the second, missingness is driven by $X_3$, and $X_3$ is similarly removed. In both cases, we vary the corresponding coefficient to control the strength of MNAR dependence. The results are shown in Figures~\ref{fig:MNARX2} and~\ref{fig:MNARX3}. These experiments exhibit qualitatively similar behaviour. In particular, performance degrades gradually as the strength of MNAR dependence increases, rather than failing abruptly. Moreover, we typically continue to observe improvements over the unweighted baseline in regimes where the dominant drivers of missingness are observed. This suggests that our proposed method shows improvement provided that the primary sources of missingness dependence are captured within the observed data.

\begin{figure}[ht]
  
    \centering
    \includegraphics[width=\textwidth]{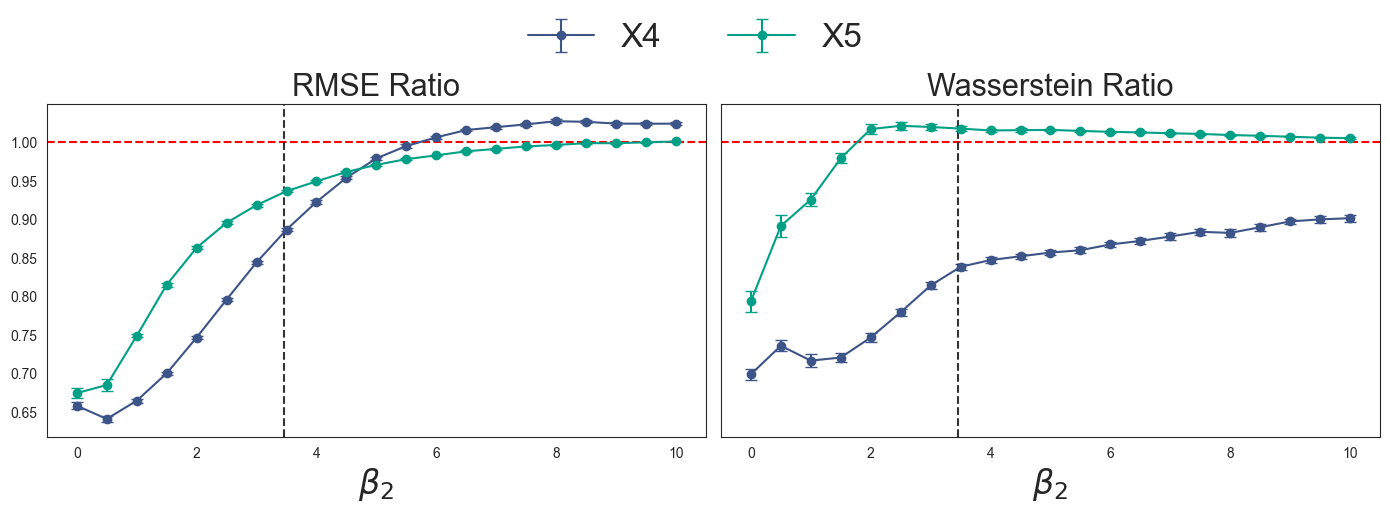}
    \captionsetup{skip=4pt}
    \caption{Values below 1 indicate improved performance. Mean performance ratios of the weighted method relative to its unweighted counterpart are shown for linear conditional models. Lines denote mean ratios, with error bars given by $\pm 1.96$ SEM, corresponding to approximate 95\% confidence intervals. The vertical line indicates the value $\beta_2^\star$, at which the majority of variance in the missingness mechanism is explained by the unobserved variable $X_2$ rather than the observed covariates. The plot is best viewed in colour.
    }\label{fig:MNARX2}
\end{figure}

\begin{figure}[ht]
    
    \centering
    \includegraphics[width=\textwidth]{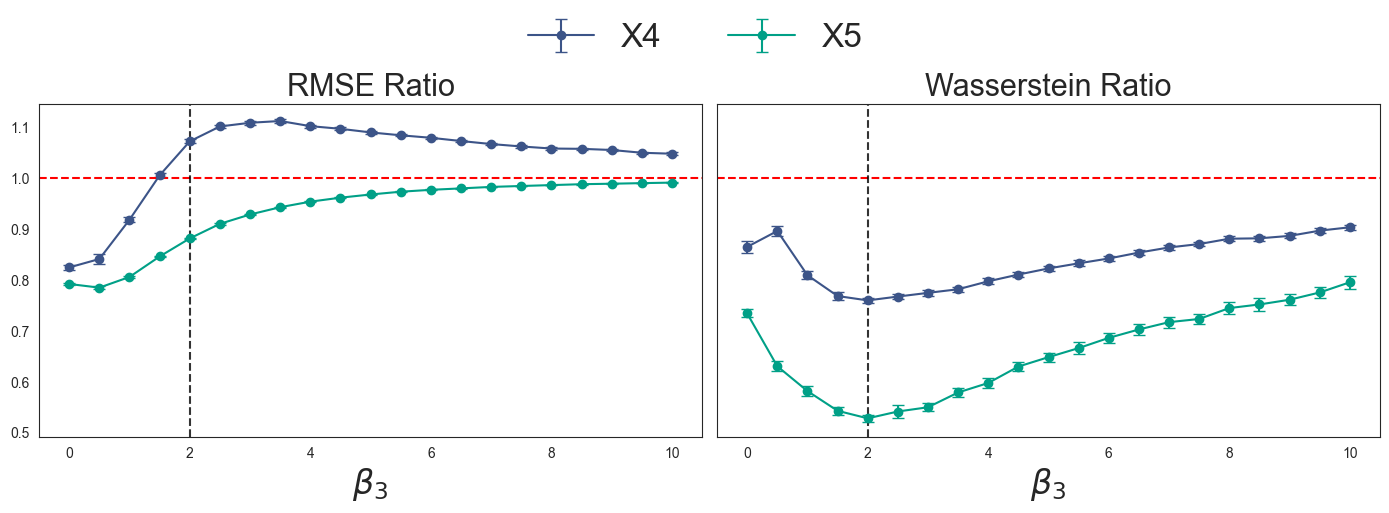}
    \captionsetup{skip=4pt}
    \caption{Values below 1 indicate improved performance. Mean performance ratios of the weighted method relative to its unweighted counterpart are shown for linear conditional models. Lines denote mean ratios, with error bars given by $\pm 1.96$ SEM, corresponding to approximate 95\% confidence intervals. The vertical line indicates the value $\beta_3^\star$, at which the majority of variance in the missingness mechanism is explained by the unobserved variable $X_3$ rather than the observed covariates. The plot is best viewed in colour.
    }\label{fig:MNARX3}
\end{figure}

\subsubsection{Misspecified Weights Model}\label{app:RobustnessToAssumptionsWeightEstimation}

We next assess robustness to misspecification of the weight model. In contrast to the MNAR experiments above, we retain a MAR mechanism but introduce controlled misspecification in the functional form used to estimate the importance weights.

\paragraph{Simulation setup:}
We generate data according to the same base process as above. Missingness in $X_4$ and $X_5$ is then introduced via logistic models depending on nonlinear transformations of the covariates:
\[
P(R_j = 1 \mid X) = \sigma\!\left(c + \beta_1 \tilde{X}_1 + \beta_2 \tilde{X}_2 + \beta_3 \tilde{X}_3 \right), \quad j \in \{4,5\},
\]
where $\tilde{X}_k$ denotes a nonlinear transformation of $X_k$ defined as
\[
\tilde{X}_k = (1 - \eta) X_k + \eta \, g(X_k), \quad k \in \{1,2,3\},
\]
with
\[
g(x) = \tfrac{1}{3}x^3 + 0.5x^2 - x - 1.
\]

The parameter $\eta \in [0,2]$ controls the degree of nonlinearity in the missingness mechanism. When $\eta = 0$, the mechanism is linear in $(X_1, X_2, X_3)$ and therefore correctly specified by a standard logistic model. As $\eta$ increases, the mechanism becomes increasingly nonlinear, inducing progressive misspecification when weights are estimated using a linear model.

\paragraph{Evaluation:}
We estimate importance weights using a logistic model linear in $(X_1, X_2, X_3)$, thereby introducing misspecification whenever $\eta > 0$. Performance is again evaluated using RMSE and Wasserstein distance, reported as ratios relative to the unweighted baseline.

Figure~\ref{fig:MissW} shows that performance degrades smoothly as $\eta$ increases. In particular, the proposed method continues to provide improvements over the unweighted baseline under mild to moderate misspecification. This suggests that our method is robust to moderate departures from correct weight specification.

\begin{figure}[ht]
  
    \centering
    \includegraphics[width=\textwidth]{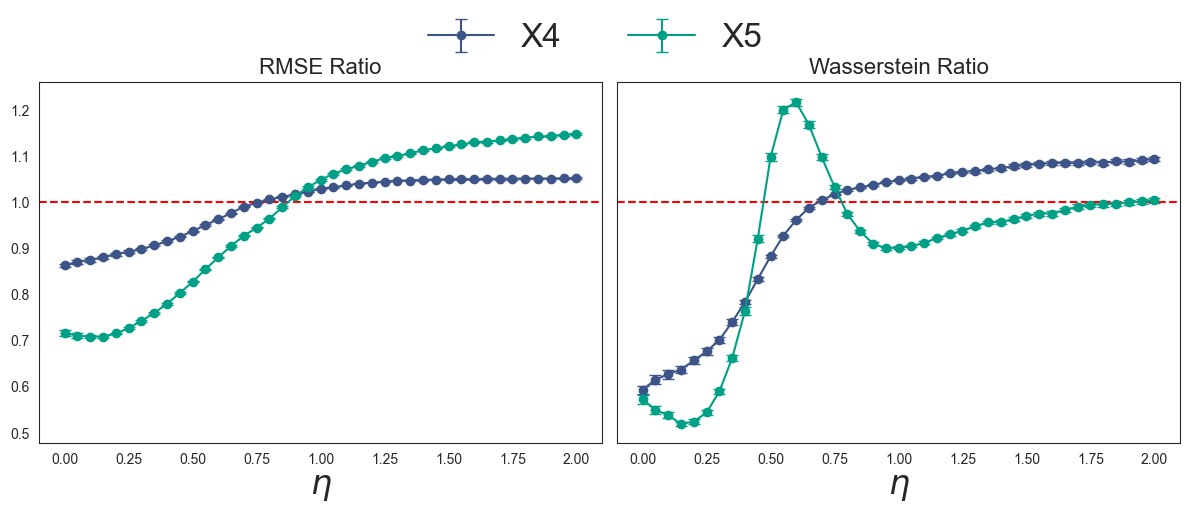}
    \captionsetup{skip=4pt}
    \caption{Values below 1 indicate improved performance. Mean performance ratios of the weighted method relative to its unweighted counterpart are shown for linear conditional models. Lines denote mean ratios across dataset settings, with error bars given by $\pm 1.96$ SEM, corresponding to approximate 95\% confidence intervals. The plot is best viewed in colour.
    }\label{fig:MissW}
\end{figure}

\subsection{Downstream Application}\label{appendix:DownstreamApplication}

We evaluate downstream performance across three tasks: downstream regression, downstream classification, and correlation matrix recovery. Below, we describe the experimental setup and results for each task.

\subsubsection{Downstream Regression}

We evaluate downstream regression on the \textbf{Gas} dataset, where \textbf{COValue} is held out as the prediction target. Missingness is introduced into the remaining features following the MAR simulation procedure outlined in Section~\ref{sec:ExperimentalSetup}. Each dataset is then imputed using the competing methods.

For each imputed dataset, we train tuned Ridge and XGBoost regression models to predict the target variable. We additionally include an ``unimputed'' baseline for comparison. For XGBoost, this corresponds to using the raw data directly, as the model can natively handle missing values. For linear models, we use mean imputation and augment inputs with missingness indicator variables.

We evaluate performance using RMSE. The results are shown in Figure~\ref{fig:DownstreamRegression}.

\paragraph{Downstream XGBoost regressor:}
For linear imputation methods, unweighted imputation can harm downstream performance, whereas our weighted approach improves it. For RF-based imputation, we observe improvements over the unimputed baseline, but insignificant difference from weighting. For MLP-based imputation, our weighted approach yields substantial improvements over both the unweighted and unimputed baselines.

\paragraph{Downstream Ridge regressor:}
For linear imputation methods, our weighted approach significantly improves performance over the unweighted variant, although both are comparable to the unimputed baseline. For RF-based imputation, we again observe improvements over the unimputed baseline, with minimal difference between weighted and unweighted methods. For MLP-based imputation, unweighted imputation can degrade performance relative to the unimputed baseline, whereas our weighted method yields substantial improvements.

\begin{figure}[ht]
  
    \centering
    \includegraphics[width=\textwidth]{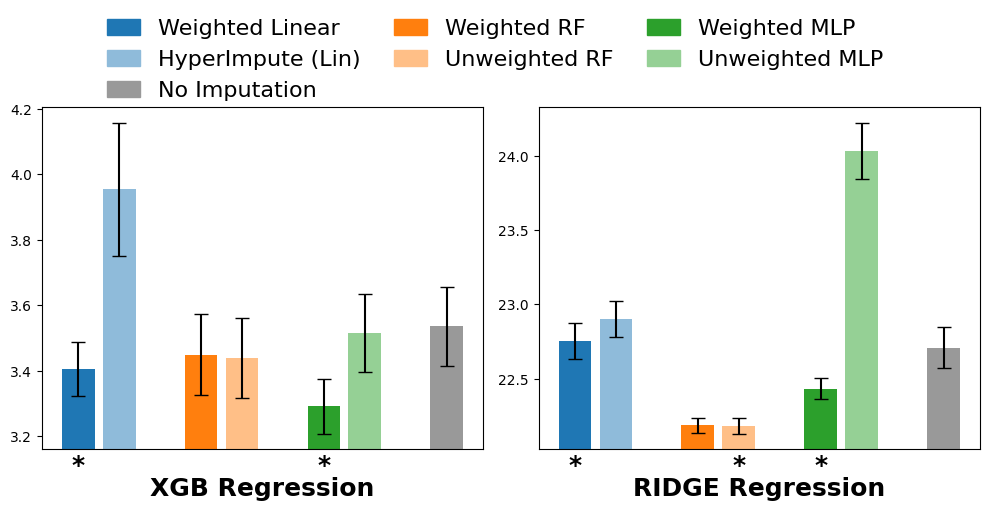}
    \captionsetup{skip=4pt}
    \caption{Lower values indicate better performance. We report mean RMSE performance of weighted and unweighted imputation methods, alongside an unimputed baseline, on downstream regression tasks evaluated on a held-out test set. Asterisks (*) indicate statistically significant improvements of weighted methods over their unweighted counterparts, based on a paired Wilcoxon signed-rank test ($p < 0.05$). Error bars denote $\pm 1.96$ SEM, corresponding to approximate 95\% confidence intervals. The plot is best viewed in colour.    }\label{fig:DownstreamRegression}
\end{figure}

\subsubsection{Downstream Classification} 

We evaluate downstream classification on the \textbf{Letter} dataset using the letter label as the prediction target. This amounts to a multiclass problem. We follow the same procedure as in the downstream regression setting, ultimately training a tuned XGBoost classifier and logistic multiclass regression model, alongside corresponding unimputed baselines.

We evaluate performance using accuracy. The results are shown in Figure~\ref{fig:DownstreamClassification}.

\paragraph{Downstream XGBoost classifier:}
We observe that the best-performing method (unweighted MLP imputation) improves accuracy by only 0.3\% over the worst-performing configuration (unweighted linear imputation), indicating that effect sizes are small in this setting.

\paragraph{Downstream logistic regression classifier:}
We observe that all imputation methods improve downstream performance over the unimputed baseline. Furthermore, our weighted approach consistently outperforms its unweighted counterparts across all imputation families.

\begin{figure}[ht]
  
    \centering
    \includegraphics[width=\textwidth]{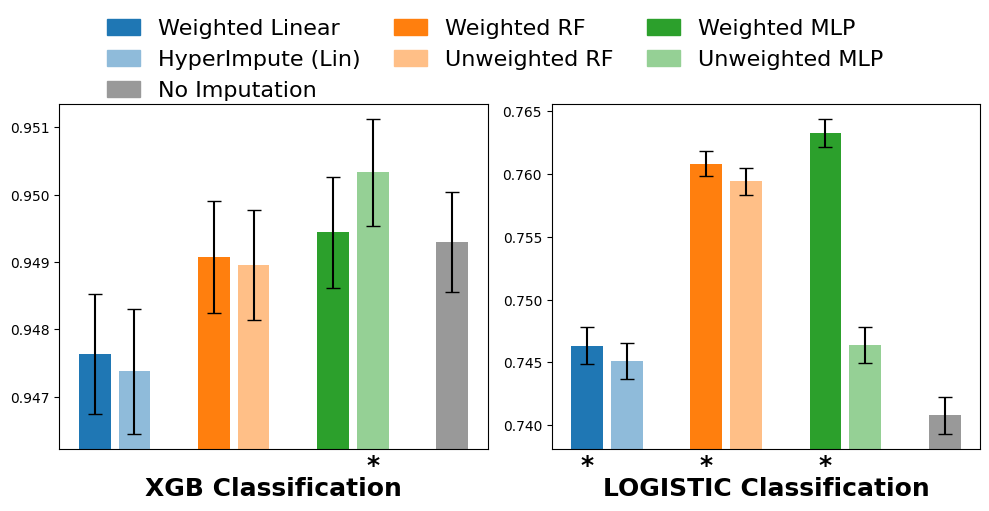}
    \captionsetup{skip=4pt}
    \caption{Higher values indicate better performance. We report mean accuracy of weighted and unweighted imputation methods, alongside an unimputed baseline, on downstream classification tasks evaluated on a held-out test set. Asterisks (*) indicate statistically significant improvements of weighted over unweighted methods based on a paired Wilcoxon signed-rank test ($p < 0.05$). Error bars denote $\pm 1.96$ SEM, corresponding to approximate 95\% confidence intervals. The plots are best viewed in colour.
    }\label{fig:DownstreamClassification}
\end{figure}

\subsubsection{Correlation Matrix Recovery}

We also evaluate our imputation methods on correlation recovery as a measure how well we can recover the joint distribution in the data. We evaluate performance by measuring the frobenius error of the correlation matrix on the imputed data against the true data. Figure~\ref{fig:DownstreamFroErr} shows that across all settings, the weighted model significantly improves over the unweighted variant.

\begin{figure}[ht]
    \centering
    \includegraphics[width=0.9\textwidth]{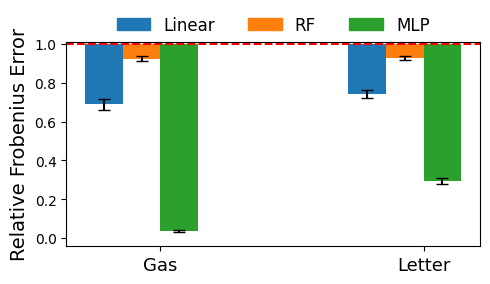}
    \captionsetup{skip=4pt}
    \caption{Values below 1 indicate improved performance. We report  mean performance ratios of the weighted method relative to the unweighted baseline across the Gas and Letter datasets and conditional model types. Bars denote mean ratios, with error bars given by $\pm 1.96$ SEM, corresponding to approximate 95\% confidence intervals. The plots are best viewed in colour.}\label{fig:DownstreamFroErr}
\end{figure}

\section{Computational Resources}\label{app:ComputationalResources}

All experiments reported in Section~\ref{sec:Experiments} were conducted on the Isambard 3 high-performance computing facility hosted by the University of Bristol. Computations were performed on CPU-only nodes; no GPU acceleration was used in any experiments.

For the main experimental results, all models were run with a maximum memory allocation of 8GB per job. This was sufficient for all reported experiments. In practice, smaller datasets required substantially less memory, and could be executed with significantly lower resource allocations.

We did not observe memory bottlenecks across any of the reported settings.



\newpage

\bibliographystyle{plainnat}   
\bibliography{references}      

\end{document}